%% file: main.tex
\documentclass[preprint,12pt]{elsarticle}

\usepackage{amssymb}
\usepackage{amsmath,amsthm}
\usepackage{graphicx}%
\usepackage{subcaption}
\usepackage{multirow}%
\usepackage{amsmath,amssymb,amsfonts}%
\usepackage{amsthm}%
\usepackage{mathrsfs}%
\usepackage[title]{appendix}%
\usepackage{xcolor}%
\usepackage{textcomp}%
\usepackage{manyfoot}%
\usepackage{booktabs}%
\usepackage{algorithm}%
\usepackage{algorithmicx}%
\usepackage{algpseudocode}%
\usepackage{listings}%
\usepackage{hyperref}%
\usepackage{tikz}
\usetikzlibrary{positioning} 
\usepackage{pgfplots}
\usepackage[colorinlistoftodos]{todonotes} 
\usepackage{bm}
\usepackage{microtype}

\newtheorem{lemma}{Lemma}

\newcommand{\sbigotimes}{%
  \mathop{\mathchoice{\textstyle\bigotimes}{\bigotimes}{\bigotimes}{\bigotimes}}%
}

\usepackage{lineno}

\journal{Journal of Computational Physics}

\begin{document}

\begin{frontmatter}



\title{Scalable Gaussian process modeling of \\ parametrized spatio-temporal fields}

\author[1]{Srinath Dama}\ead{srinath.dama@mail.utoronto.ca}

\author[1]{Prasanth B. Nair\corref{cor1}}\ead{pbn@utias.utoronto.ca}


\cortext[cor1]{Corresponding author}

\affiliation[1]{organization={University of Toronto Institute for Aerospace Studies}, 
            addressline={4925 Dufferin Street}, 
            city={Toronto},
            postcode={M3H 5T6}, 
            state={Ontario},
            country={Canada}}

\begin{abstract}
We introduce a scalable Gaussian process (GP) framework with deep product kernels for data-driven learning of parametrized spatio-temporal fields  over fixed or parameter-dependent domains. The proposed framework learns a continuous  representation,  enabling predictions at arbitrary spatio-temporal coordinates, independent of the training data resolution.  
We leverage Kronecker matrix algebra to formulate a computationally efficient training procedure with complexity that scales nearly linearly with the total number of spatio-temporal grid points.
A key feature of our approach is the efficient computation of the posterior variance at essentially the same computational cost as the posterior mean (exactly for Cartesian grids and via rigorous bounds for unstructured grids), thereby enabling scalable uncertainty quantification. 
Numerical studies on a range of benchmark problems demonstrate that the proposed method achieves accuracy competitive with  operator learning methods such as Fourier neural operators and deep operator networks. On the one-dimensional unsteady Burgers' equation, our method surpasses the accuracy of projection-based reduced-order models. These results establish the proposed framework as an effective tool for data-driven surrogate modeling, particularly when uncertainty estimates are required for downstream tasks.
\end{abstract}







\end{frontmatter}



\pagebreak 

\input{chapters/Intro}

\input{chapters/Methods}

\input{chapters/Results}

\input{chapters/Conclusions}

\input{chapters/Appendix}

\end{document}

%% file: chapters/Intro.tex
\section{Introduction}

Many-query problems are ubiquitous in computational science and engineering, arising in contexts such as uncertainty quantification, Bayesian inverse problems, and simulation-based optimization~\cite{,Peherstorfer2018,forrester2008,  keane2005computational}. When each query requires evaluation of a high-fidelity simulation, the associated computational cost can become prohibitive. This has motivated the development of surrogate models that use samples of the simulation output to construct approximations of the input-output map that can be evaluated at significantly lower computational cost.

When constructing surrogate models of parametrized partial differential equations (PDEs), a key challenge is the high dimensionality of the solution space. 
A widely adopted strategy for addressing this challenge is the   two-step reduced-order modeling (ROM) approach. In the first step, a dimensionality reduction method is applied to an ensemble of high-dimensional solution snapshots to construct a low-dimensional latent representation. The effectiveness of this step is fundamentally governed by the Kolmogorov $n$-width of the solution manifold, which quantifies the best achievable approximation error over all $n$-dimensional subspaces~\cite{Peherstorfer2022}. For problems with rapidly decaying 
$n$-width, linear manifold representations such as principal component analysis (PCA) have proven highly effective~\cite{lucia2004reduced,  couplet2005calibrated,  veroy2005certified, bui-thanh2008, audouze2009reduced, quarteroni2011certified, hesthaven2016certified, hesthaven2018non}. However, many problems of practical interest exhibit slowly decaying 
$n$-width, which has motivated nonlinear dimensionality reduction approaches including Isomap~\cite{xing2015reduced}, kernel PCA~\cite{xing2016manifold}, diffusion maps~\cite{xing2016manifold}, and autoencoder networks~\cite{xu2020multi}.

The second step involves constructing a mapping from input parameters to the latent representation. Intrusive methods~\cite{lucia2004reduced, veroy2005certified, quarteroni2011certified, hesthaven2016certified} achieve this by projecting the governing equations onto the low-dimensional subspace, ensuring  consistency with the underlying physics but requiring direct access to the full-order model. Non-intrusive methods instead employ function approximation techniques such as radial basis function interpolation~\cite{audouze2013nonintrusive, berzicnvs2020standardized}, Gaussian processes~\cite{higdon2008,xing2015reduced, xing2016manifold, guo2018reduced}, or neural networks~\cite{hesthaven2018non,renganathan2021enhanced, renganathan2020machine, xu2020multi} to learn the reduced-order map directly from data.

A fundamental limitation of the two-step reduced-order modeling approach is that the dimension-reduction map is typically learned in an unsupervised manner, without accounting for the dependence of the snapshots on the input parameters. For parametrized problems, this can yield latent coordinates that are not injective with respect to the full state; consequently, distinct full states (often corresponding to different parameters or unresolved modes) map to the same latent point. In that case, the induced latent time derivatives are generally multi-valued, and no autonomous reduced-order ODE on the latent space can reproduce the projected dynamics exactly~\cite{course2023amortized}.

In recent years, operator learning methods have emerged as a popular approach for learning solution maps of PDEs. Methods such as Fourier neural operators (FNO)~\cite{li2020fourier, kovachki2023neural}, DeepONets~\cite{lu2021learning, lu2022comprehensive,goswami2022deep}, and Gaussian process  operator models~\cite{mora2025operator} 
aim to learn mappings between infinite-dimensional function spaces and have demonstrated success on a range of benchmark problems. Physics-informed variants that incorporate governing equations into the training objective have also been developed in the literature~\cite{li2024physics, wang2021learning}. 
While operator learning architectures provide flexible nonparametric function approximators, they are typically trained as deterministic predictors and require additional mechanisms to quantify predictive uncertainty. Moreover, many applications require well-calibrated uncertainty estimates together with continuous predictions over space and time. These considerations motivate the development of probabilistic approaches that can  handle large spatio-temporal datasets arising in high-fidelity simulations.

Gaussian process (GP) regression offers an elegant  probabilistic framework for surrogate modeling, with uncertainty estimates that are critical for downstream decision-making tasks~\cite{kennedy2000predicting,williams2006gaussian,higdon2008,forrester2008}. However, standard GP regression requires $\mathcal{O}(n^3)$ operations for $n$ training points, rendering direct application to high-dimensional spatio-temporal fields intractable. 
A variety of scalable GP methods have been developed to address this bottleneck, including sparse inducing-point approximations~\cite{titsias2009variational, hensman2013gaussian} and structured kernel interpolation~\cite{wilson2015kernel}; see \cite{liu2020} for a review. For certain classes of problems
(e.g., when the training inputs lie on a Cartesian grid), product-structured kernels enable the use of Kronecker matrix algebra to efficiently perform exact GP inference for large datasets; see, for example, \cite{saatcci2012scalable, Bilionis2013, wilson2014fast,evans2018exploiting}. Despite this progress, the computational challenges associated with scalable GP modeling of parametrized spatio-temporal fields remain largely unresolved. The key challenge is that the high-resolution unstructured spatial grids encountered in many real-world  applications preclude direct use of Kronecker matrix algebra.

In this paper, we develop a scalable GP framework for learning parametrized spatial and spatio-temporal fields over fixed or  parametrized spatial domains. Our approach uses deep product-structured kernels and Kronecker matrix algebra to achieve scalability. 
We introduce our methodology by considering the setting where the spatial grid has a Cartesian product structure. We subsequently  generalize our framework to handle general unstructured spatial grids and parametrized spatial domains, broadening its applicability to a wide class of problems. 
The generalization to unstructured grids is enabled by a gappy-grid formulation in which observations on an incomplete rectilinear grid are augmented with pseudovalues chosen so that the resulting linear system retains full Kronecker structure; we prove (Lemma~\ref{lem:fill_gaps_equivalence}) that the solution on the observed points is exact, i.e., identical to that obtained from the original unstructured system. A key challenge in this setting is the computation of the posterior variance, since the covariance matrix restricted to the observed points lacks Kronecker structure. We address this by establishing rigorous upper and lower bounds on the posterior variance (Lemma~\ref{lem:gappy_posterior_variance_bounds}) that are computable via Kronecker algebra at a cost comparable to that of the posterior mean.

The key contributions of the present work are: (1)~a deep product kernel that captures nonlinear correlations across parameters, spatial coordinates, and time; (2)~a training algorithm with complexity that scales nearly linearly with the number of spatio-temporal grid points; (3)~computational strategies for dealing with Cartesian and unstructured spatial grids as well as  parametrized domains; and (4)~efficient computation of the posterior variance at the same cost as the posterior mean, providing exact uncertainty quantification for Cartesian spatial grids and rigorous theoretical bounds for unstructured grids.

We validate the proposed framework on a diverse set of fluid and solid mechanics benchmarks involving rectilinear and unstructured grids, including problems with parametrized geometries. Across these cases, the proposed method matches or exceeds the accuracy of established operator-learning baselines such as FNO and DeepONet, while providing posterior uncertainty estimates at comparable computational cost. We additionally demonstrate favorable performance relative to projection-based reduced-order models on the unsteady Burgers' benchmark.

\section{Preliminaries} \label{section:Problem_statement}

This section formalizes the learning problem considered in this work: the construction of a data-driven surrogate for parametrized spatio-temporal fields. We then briefly review Gaussian process regression and highlight the computational bottlenecks that arise when standard GP methodology is applied at the spatio-temporal resolutions typical of scientific simulations.

\subsection{Problem Statement}

Let $\Omega \subset \mathbb{R}^d$ ($d\in\{1,2,3\}$) denote a spatial domain, $\mathcal{T}=[0,T]$ a time interval, and $\mathcal{P}\subset\mathbb{R}^D$ a parameter domain. For each $(\boldsymbol{\mu},t)\in\mathcal{P}\times\mathcal{T}$, let
$u(\cdot\,,t\,; \boldsymbol{\mu})\in \mathcal{V}$ 
denote a scalar field over $\Omega$, where $\mathcal{V}$ is a Banach space of real-valued functions on $\Omega$ (e.g., $\mathcal{V}=L^2(\Omega)$ or $H^1(\Omega)$). Equivalently, $u$ may be viewed pointwise as a map
$
u:\mathcal{P}\times\Omega\times\mathcal{T}\to\mathbb{R}$, 
$(\boldsymbol{\mu},\mathbf{x},t)\mapsto u(\mathbf{x},t; \boldsymbol{\mu})$.

Our goal is to learn an approximation of the parametrized spatio-temporal field as a function of the parameter vector $\boldsymbol{\mu}$, spatial coordinates $\mathbf{x}$, and time $t$ using observations of $u$. We denote the training dataset as
$$
\mathcal{D}:= \left \{ \left ( \boldsymbol{\mu}^{(i)}, \, t^{(j)}, \, u(\cdot \,,t^{(j)}\,; \boldsymbol{\mu}^{(i)})\right  ) \right \}_{i=1,2,\ldots,N;\,j=1,2,\ldots,N_t},
$$
where $\boldsymbol{\mu}^{(i)} \in \mathcal{P}$, $t^{(j)} \in [0, T]$,  and  $  u( \cdot \,, t^{(j)}\,; \boldsymbol{\mu}^{(i)}) \in \mathcal{V}$. 
Given $\mathcal{D}$, our goal is to construct a surrogate model of $u$ that can be evaluated at arbitrary $(\boldsymbol{\mu},\mathbf{x},t)\in\mathcal{P}\times\Omega\times\mathcal{T}$.
The learning problem we consider commonly arises in the context of surrogate modeling of parametrized, time-dependent PDEs of the form
\begin{equation}
\frac{\partial u(\mathbf{x},t; \boldsymbol{\mu})}{\partial t}+\mathcal{N}\bigl(u(\mathbf{x},t; \boldsymbol{\mu})\bigr)=f(\mathbf{x},t; \boldsymbol{\mu}),
\qquad (\boldsymbol{\mu},\mathbf{x},t)\in\mathcal{P}\times\Omega\times\mathcal{T},
\label{unsteady-pde}
\end{equation}
where $\mathcal{N}$ is a parametrized linear or nonlinear operator and $f$ denotes a source term. In practice, each field $u(\cdot\,,t^{(j)}\,; \boldsymbol{\mu}^{(i)})$ is available only through a spatial discretization, yielding a snapshot vector $\mathbf{y}^{(i,j)}\in\mathbb{R}^M$ evaluated over the spatial grid points $\{\mathbf{x}^{(k)}\}_{k=1}^M\subset\Omega$. In the present work, we consider the problem of approximating spatio-temporal fields over fixed as well as parametrized spatial domains.  

\subsection{Background on Gaussian process regression} \label{section:GPR}

In what follows we work with the pointwise data obtained by enumerating all parameter, time, and spatial grid combinations, yielding $NMN_t$ scalar observations. We introduce the combined input $
\mathbf{z} = [\boldsymbol{\mu}, \mathbf{x}, t]^T \in \mathbb{R}^{D+d+1}$, 
and interpret $u(\mathbf{z})$ as the pointwise evaluation $u(\mathbf{x},t; \boldsymbol{\mu})$ of the field at input $\mathbf{z}$. GP regression~\cite{williams2006gaussian} provides a probabilistic surrogate model by placing a prior distribution on $u(\cdot)$ of the form
\[
u(\mathbf{z}) \sim \mathcal{GP} \bigl(m(\mathbf{z}), k(\mathbf{z},\mathbf{z}')\bigr),
\]
where $m$ and $k$ denote a mean function and a positive definite covariance (kernel) function, respectively, that are parametrized by hyperparameters $\boldsymbol{\theta}$. 

We assume that the pointwise observations are corrupted by Gaussian noise, i.e.,
$
y^{(i)} = u(\mathbf{z}^{(i)}) + \epsilon^{(i)}$, where $\epsilon^{(i)} \sim \mathcal{N}(0,\sigma_n^2)$. We collect the inputs and outputs into:
\[
\mathbf{Z} = \left[\mathbf{z}^{(1)}, \mathbf{z}^{(2)}, \ldots, \mathbf{z}^{(NMN_t)}\right]^T \in \mathbb{R}^{NMN_t \times (D+d+1)}, 
\]
and 
\[
\mathbf{y} = \left[y^{(1)}, y^{(2)}, \ldots, y^{(NMN_t)}\right]^T \in \mathbb{R}^{NMN_t}.
\]
Let $\mathbf{m} = \bigl[m(\mathbf{z}^{(1)}),\ldots,m(\mathbf{z}^{(NMN_t)})\bigr]^T$ and let $\mathbf{K_Z}\in\mathbb{R}^{NMN_t\times NMN_t}$ be the kernel matrix with entries $(\mathbf{K_Z})_{ij}=k(\mathbf{z}^{(i)},\mathbf{z}^{(j)})$. Under the Gaussian noise model, the covariance of $\mathbf{y}$ is given by 
$
\mathbf{K_y} = \mathbf{K_Z} + \sigma_n^2 \mathbf{I}$.
 The posterior predictive mean and variance corresponding to a test input
$\mathbf{z}_*$ are 
given by~\cite{williams2006gaussian}
\begin{equation}
\widehat{u}(\mathbf{z}_*) = m(\mathbf{z}_*) + \mathbf{k}(\mathbf{z}_*)^T \mathbf{K}_{\mathbf{y}}^{-1}\,(\mathbf{y}-\mathbf{m}) 
\label{pred_mean_orig}
\end{equation}
and 
\begin{equation}
\widehat{\sigma}^2(\mathbf{z}_*) = k(\mathbf{z}_*,\mathbf{z}_*) - \mathbf{k}(\mathbf{z}_*)^T \mathbf{K}_{\mathbf{y}}^{-1}\,\mathbf{k}(\mathbf{z}_*), 
\label{pred_dist}
\end{equation}
respectively, 
where $\mathbf{k}(\mathbf{z}_*) = [k(\mathbf{z}_*,\mathbf{z}^{(1)}),\ldots,k(\mathbf{z}_*,\mathbf{z}^{(NMN_t)})]^T$.

The kernel hyperparameters $\boldsymbol{\theta}$ are commonly learned by minimizing the negative log marginal likelihood (NLML),
\begin{eqnarray}
\textrm{NLML}(\boldsymbol{\theta})
= -\log p(\mathbf{y}\,|\,\mathbf{Z},\boldsymbol{\theta})
& = & \frac{1}{2}(\mathbf{y}-\mathbf{m})^T 
\mathbf{K}_{\mathbf{y}}^{-1} (\mathbf{y}-\mathbf{m}) 
+  \frac{1}{2}\log|\mathbf{K}_{\mathbf{y}}| \nonumber \\
& + & \frac{NMN_t}{2}\log(2\pi).
\label{NLML}
\end{eqnarray}
When using gradient-based optimization, we require the derivatives of NLML with respect to the hyperparameters, which take the form
\begin{equation}
\frac{\partial \textrm{NLML}(\boldsymbol{\theta})}{\partial \theta_j}
= 
-\frac{1}{2}(\mathbf{y}-\mathbf{m})^T \mathbf{K}_{\mathbf{y}}^{-1}\frac{\partial \mathbf{K}_{\mathbf{y}}}{\partial \theta_j}\mathbf{K}_{\mathbf{y}}^{-1}(\mathbf{y}-\mathbf{m}) + \frac{1}{2}\mathrm{tr}\!\left(\mathbf{K}_{\mathbf{y}}^{-1}\frac{\partial \mathbf{K}_{\mathbf{y}}}{\partial \theta_j}\right).
\label{NLMLderivate}
\end{equation}

The terms dominating the computational cost associated with calculating the log marginal likelihood and its derivatives include the linear solve $\mathbf{K}^{-1}_{\mathbf{y}} \mathbf{y}$, the log determinant $\log|\mathbf{K_y}|$, and the trace term $\mathrm{tr} ( \mathbf{K}^{-1}_{\mathbf{y}} \frac{\partial \mathbf{K_y}}{\partial \theta_j} )$. For dense covariance matrices, direct Cholesky factorization requires $\mathcal{O}((NMN_t)^3)$ operations and $\mathcal{O}((NMN_t)^2)$ memory. Therefore, the computational complexity associated with evaluating  \eqref{NLML} and \eqref{NLMLderivate} is $\mathcal{O}((NMN_t)^3)$. 
Given a factorization of $\mathbf{K_y}$ and a single test point $\mathbf{z}_*$, evaluating the posterior mean, $\widehat{u}(\mathbf{z}_*)$, requires $\mathcal{O}(NMN_t)$ operations, whereas evaluating the posterior variance, $\widehat{\sigma}^2(\mathbf{z}_*)$, involves a triangular solve and an inner product, requiring $\mathcal{O}((NMN_t)^2)$ operations.

In real-world applications, where training datasets are generated from high-fidelity simulations, the number of spatial grid points $M$ is significantly greater than the number of parameter samples $N$ and time steps $N_t$. For example, if $M \sim \mathcal{O}(10^6)$, $N \sim \mathcal{O}(10^3),$ and $N_t \sim \mathcal{O}(10^3)$, the computational complexity is $\mathcal{O}(10^{36})$. This exorbitant computational cost renders standard GP regression impractical at the spatio-temporal resolutions encountered in many applications.

A number of approaches have been proposed in the literature to address the computational challenges associated with GP regression for large datasets.
One way to scale GP regression for large datasets is to use sparse GP methods~\cite{titsias2009variational, hensman2013gaussian, matthews2016sparse}, which approximate the full GP with a smaller set of inducing points. Another approach is to leverage the structure of the covariance matrix when the input data has particular structure. For instance, in  one-dimensional problems with evenly spaced inputs, the covariance matrix inherits a Toeplitz structure~\cite{cunningham2008fast}. Saat{\c{c}}i et al.~\cite{saatcci2012scalable} showed that when the training data is given on a Cartesian/rectilinear grid, choosing a product kernel ensures the covariance matrix inherits a Kronecker product structure, thereby enabling significant reduction in computational complexity for large datasets. For problems where the data is given on a general unstructured grid, Wilson et al.~\cite{wilson2015kernel} introduced a scalable GP framework called KISS-GP, which is based on structured kernel interpolation for approximating the GP kernel function. This approach uses iterative methods and heavily relies on fast matrix-vector products for scalability.    

Despite significant progress in scalable GP modeling, the challenges associated with learning parametrized spatio-temporal fields have received comparatively little attention in the literature. 
In the next section, we introduce deep product-structured kernels that greatly enhance the ability of GPs to learn complex parameter-dependent spatio-temporal fields. This class of kernels together with the numerical methods based on Kronecker matrix algebra introduced in subsequent sections enable dramatic reductions in the computational complexity of training and inference when learning GP models of parametrized spatio-temporal fields.

%% file: chapters/Methods.tex
\section{Deep product kernels} \label{section:Scalable-GPR}

GPs often employ stationary covariance functions such as the squared exponential (SE) or Mat\'ern kernels~\cite{williams2006gaussian}. Stationarity means that the covariance depends on inputs only through their relative displacement, i.e., $k(\mathbf{z},\mathbf{z}') = \kappa(\mathbf{z}-\mathbf{z}')$ (and, in the isotropic case, through $\|\mathbf{z}-\mathbf{z}'\|$). While stationary kernels are effective in many settings, they impose a homogeneous correlation structure across the input space and can be insufficient for strongly nonstationary responses, for example when smoothness, anisotropy, or characteristic lengthscales vary with the inputs.

Deep kernel learning addresses this limitation by composing a stationary base kernel with a learned nonlinear feature map~\cite{wilson2016deep, calandra2016manifold}. Specifically, given a feature map $G$ (typically implemented by a neural network with parameters $\boldsymbol{\theta}_{\mathrm{nn}}$), a deep kernel is defined as
\begin{equation}
    \tilde{k}(\mathbf{z}, \mathbf{z}') = k\!\left(G(\mathbf{z}),\, G(\mathbf{z}')\right),
    \label{deepkernel}
\end{equation}
where $k$ is a stationary base kernel with hyperparameters $\boldsymbol{\theta}_{\mathrm{gp}}$. The resulting covariance function $\tilde{k}$ is generally nonstationary with respect to the original inputs $\mathbf{z}$ while remaining positive definite by construction. We denote the full set of deep-kernel parameters by $\boldsymbol{\theta}= \{\boldsymbol{\theta}_{\mathrm{nn}}, \boldsymbol{\theta}_{\mathrm{gp}}\}$.

In the present work, we propose a \emph{deep product kernel} (DPK) that factorizes across the parameter, spatial, and temporal components of the input  $\mathbf{z}=(\boldsymbol{\mu},\mathbf{x},t)$ as follows:
\begin{equation}
    \begin{split}
        k_{\scriptstyle{DPK}}(\mathbf{z}^{(i)}, \mathbf{z}^{(j)})
        &= \overbrace{k_{\boldsymbol{\mu}} \!\left(\tilde{\boldsymbol{\mu}}^{(i)}, \tilde{\boldsymbol{\mu}}^{(j)}\right)}^{\textrm{parameter kernel}}
        \;\times\;
        \overbrace{k_{\mathbf{x}} \!\left(\tilde{\mathbf{x}}^{(i)}, \tilde{\mathbf{x}}^{(j)}\right)}^{\textrm{spatial kernel}}
        \;\times\;
        \overbrace{k_{t}\!\left(\tilde{t}^{(i)}, \tilde{t}^{(j)}\right)}^{\textrm{temporal kernel}}
        \\
        &=  k_{\boldsymbol{\mu}} \left(G_{\boldsymbol{\mu}}(\boldsymbol{\mu}^{(i)}), G_{\boldsymbol{\mu}}(\boldsymbol{\mu}^{(j)}) \right)  \times k_{\mathbf{x}} \left(G_{\mathbf{x}}(\mathbf{x}^{(i)}), G_{\mathbf{x}}(\mathbf{x}^{(j)}) \right) \\
            & \qquad \qquad \qquad \qquad \times k_t \left(G_t(t^{(i)}), G_t(t^{(j)}) \right), 
        \label{product_kernel}
    \end{split}
\end{equation}
where $\tilde{\boldsymbol{\mu}}=G_{\boldsymbol{\mu}}(\boldsymbol{\mu})$, $\tilde{\mathbf{x}}=G_{\mathbf{x}}(\mathbf{x})$, and $\tilde{t}=G_t(t)$ are component-wise nonlinear transformations given by neural networks, and $k_{\boldsymbol{\mu}}, k_{\mathbf{x}}, k_t$ are stationary base kernels on the corresponding latent spaces. Replacing each $G$ with the identity map reduces \eqref{product_kernel} to a standard stationary product kernel. Parameters for the deep product kernel are denoted as
$\boldsymbol{\theta}= \{(\boldsymbol{\theta}_{\mathrm{nn}}^{\boldsymbol{\mu}}, \boldsymbol{\theta}_{\mathrm{gp}}^{\boldsymbol{\mu}}),\,
(\boldsymbol{\theta}_{\mathrm{nn}}^{\mathbf{x}}, \boldsymbol{\theta}_{\mathrm{gp}}^{\mathbf{x}}),\,
(\boldsymbol{\theta}_{\mathrm{nn}}^{t}, \boldsymbol{\theta}_{\mathrm{gp}}^{t}) \}$. 

The product structure in \eqref{product_kernel} imposes separability of the covariance across $(\boldsymbol{\mu},\mathbf{x},t)$. This kernel choice induces Kronecker structure in the covariance matrix when the training data lies on a Cartesian product grid, which is essential for computational scalability. The factorization can therefore be viewed as an inductive bias that enables scalable inference while still allowing rich, input-dependent representations through the component feature maps. In particular, although the covariance is separable, the resulting GP prior over functions is not restricted to strictly separable functions. To illustrate this point, consider the Cartesian product input space $\mathcal{Z}=\mathcal{M}\times\mathcal{X}\times\mathcal{T}$ with $\mathbf{z}=(\boldsymbol{\mu},\mathbf{x},t)$ and $\mathbf{z}'=(\boldsymbol{\mu}',\mathbf{x}',t')$. The DPK in \eqref{product_kernel} can be written as
\[
k_{\mathrm{DPK}}(\mathbf{z},\mathbf{z}')
=
k_{\boldsymbol{\mu}} \left(G_{\boldsymbol{\mu}}(\boldsymbol{\mu}),G_{\boldsymbol{\mu}}(\boldsymbol{\mu}')\right)\,
k_{\mathbf{x}} \left(G_{\mathbf{x}}(\mathbf{x}),G_{\mathbf{x}}(\mathbf{x}')\right)\,
k_t \left(G_t(t),G_t(t')\right),
\]
where $G_{\boldsymbol{\mu}},G_{\mathbf{x}},G_t$ are component-wise feature maps (neural networks) and $k_{\boldsymbol{\mu}},k_{\mathbf{x}},k_t$ are stationary base kernels. To make the connection with a standard deep kernel explicit, consider the common case in which each factor is a squared exponential (SE) kernel, i.e.,
$k_{\boldsymbol{\mu}}(u,u')=\exp (-\|L_{\boldsymbol{\mu}}(u-u')\|_2^2 )$,
$k_{\mathbf{x}}(v,v')=\exp (-\|L_{\mathbf{x}}(v-v')\|_2^2 )$, and 
$
k_t(w,w')=\exp (-\|L_t(w-w')\|_2^2)$, where $L_{\boldsymbol{\mu}}, L_{\mathbf{x}}, L_t$ are matrices encoding the length-scale hyperparameters of the component kernels.

Multiplying the exponentials and collecting terms, the DPK can be expressed as 
$k_{\mathrm{DPK}}(\mathbf{z},\mathbf{z}')
=\exp (- \bigl\|\boldsymbol{\Phi}(\mathbf{z})-\boldsymbol{\Phi}(\mathbf{z}')\bigr\|_2^2 )$, 
where the joint latent map is the block concatenation
$\boldsymbol{\Phi}(\mathbf{z})
= [L_{\boldsymbol{\mu}}G_{\boldsymbol{\mu}}(\boldsymbol{\mu}),~
L_{\mathbf{x}}G_{\mathbf{x}}(\mathbf{x}),~
L_t G_t(t) ]^T$. 
That is, in the widely used SE setting, $k_{\mathrm{DPK}}$ is exactly an SE deep kernel applied to the concatenated feature space maps $\boldsymbol{\Phi}(\cdot)$. From this perspective, the product structure is not an intrinsic expressivity limitation relative to deep kernels; it is a structural constraint on the feature space mapping, introduced to induce the Kronecker algebra needed for scalability. 

With sufficiently flexible feature maps and standard universal base kernels on compact domains (including SE and many Mat\'ern choices), product kernels retain universal approximation capacity on the product space~\cite{micchelli2006universal, steinwart2001influence}. Fully non-separable deep kernels can be more data-efficient in regimes dominated by strong parameter-space-time interactions, but they are typically computationally prohibitive at the spatial resolutions considered here. In Section~\ref{sec_results} we empirically assess this accuracy--scalability trade-off.

For training datasets arising from numerical solution of parametrized PDEs, the settings for the parameters $\boldsymbol{\mu}$ are typically selected via a design of computer experiments approach~\cite{santner2003design}, and spatial locations $\mathbf{x}$ correspond to a $d$-dimensional spatial grid, where $d \in \{1,2,3\}$. The full input set can be expressed as a Cartesian product
$$
\{\mathbf{z}^{(i)}\}_{i=1}^{NMN_t}
=
\{\boldsymbol{\mu}^{(j)}\}_{j=1}^{N}
\times
\{\mathbf{x}^{(k)}\}_{k=1}^{M}
\times
\{t^{(\ell)}\}_{\ell=1}^{N_t}.
$$
When the training inputs are ordered consistently with this product structure, the resulting covariance matrix inherits a Kronecker product structure~\cite{saatcci2012scalable}:
\begin{equation}
    \mathbf{K}_{\mathbf{Z}}
    \;=\;
    \mathbf{K}_{\boldsymbol{\mu}}
    \otimes
    \mathbf{K}_{\mathbf{X}}
    \otimes
    \mathbf{K}_{\mathbf{T}},
    \label{kron_product-1}
\end{equation}
where
$\mathbf{X} = [\mathbf{x}^{(1)}, \mathbf{x}^{(2)}, \ldots, \mathbf{x}^{(M)} ]^T \in \mathbb{R}^{M \times d}$,
$\mathbf{T} = [t^{(1)}, t^{(2)}, \ldots, t^{(N_t)}]^T \in \mathbb{R}^{N_t}$, 
$\mathbf{K}_{\boldsymbol{\mu}} =
\{ k_{\boldsymbol{\mu}} (\tilde{\boldsymbol{\mu}}^{(i)}, \tilde{\boldsymbol{\mu}}^{(j)} ) \}_{i,j=1}^{N}
\in \mathbb{R}^{N\times N}$, 
$\mathbf{K}_{\mathbf{X}} =
\{ k_{\mathbf{x}} (\tilde{\mathbf{x}}^{(i)}, \tilde{\mathbf{x}}^{(j)} ) \}_{i,j=1}^{M}
\in \mathbb{R}^{M\times M}$, and 
$\mathbf{K}_{\mathbf{T}} =
\{ k_t (\tilde{t}^{(i)}, \tilde{t}^{(j)} ) \}_{i,j=1}^{N_t}
\in \mathbb{R}^{N_t\times N_t}$ 
denote the parameter, spatial, and temporal covariance matrices, respectively.

Despite the Kronecker factorization \eqref{kron_product-1}, the spatial covariance matrix $\mathbf{K}_{\mathbf{X}}$ can remain too large to store and factorize when $M \sim \mathcal{O}(10^6)$, limiting the application of exact GP training. The following two sections address this challenge. First, we consider problems where the spatial points are arranged on a complete Cartesian product grid. Subsequently, we address the general case of  unstructured spatial grids. In both settings, we develop scalable approaches for exact GP training at the million-point spatial scale.

\section{Gaussian process modeling of parametrized spatio-temporal fields on rectilinear spatial grids}
\label{section:Struct-GPR}

We now present the core computational ingredients of the proposed approach in the setting where the spatial locations lie on a rectilinear (Cartesian product) grid. In practical applications, the spatial covariance matrix is significantly larger than the parameter and temporal covariance matrices. This motivates factorization of the spatial kernel across spatial dimensions to induce additional Kronecker structure, enabling scalable GP training and inference for large-scale datasets. 

\subsection{Rectilinear spatial grids and Kronecker-structured covariance}

We consider a $d$-dimensional spatial coordinate, $\mathbf{x}\in\mathbb{R}^d$, with the training points arranged on a rectilinear grid, i.e., 
$\{\mathbf{x}^{(i)}\}_{i=1}^{M}= \mathbf{x}_1\times\cdots\times\mathbf{x}_d$,
where $\mathbf{x}_\ell=\{x_\ell^{(j)}\}_{j=1}^{M_\ell}$ denotes the set of distinct coordinates along the $\ell$th spatial dimension, and the total number of spatial points is
$M=\prod_{\ell=1}^{d} M_\ell$. 
In this setting, we further factorize the spatial kernel appearing in the deep product kernel \eqref{product_kernel} across spatial coordinates as follows: 
\begin{equation}
\begin{split}
k_{\mathbf{x}}(\mathbf{x}^{(i)},\mathbf{x}^{(j)})
&=
\prod_{\ell=1}^{d} k_{x_\ell}(\tilde{x}_\ell^{(i)},\tilde{x}_\ell^{(j)}),
\qquad
\tilde{x}_\ell = G_{x_\ell}(x_\ell),
\end{split}
\label{spatial_product_kernel}
\end{equation}
where $G_{x_\ell}: \mathbb{R} \to \mathbb{R}^{\tilde{d}_\ell}$ denotes a nonlinear feature map (parametrized by a neural network) that lifts the scalar coordinate 
$x_\ell$ to a higher-dimensional latent space of dimension 
$\tilde{d}_\ell \geq 1$. Consequently, $k_{x_\ell}: \mathbb{R}^{\tilde{d}_\ell} \times \mathbb{R}^{\tilde{d}_\ell} \to \mathbb{R}$ acts as a stationary base kernel on this high-dimensional feature space, allowing the model to capture complex nonstationary correlations along each spatial dimension.

Since $\mathbf{X} = [\mathbf{x}^{(1)}, \mathbf{x}^{(2)}, \ldots, \mathbf{x}^{(M)}]^T$ is a Cartesian product grid and the kernel is separable across spatial coordinates, the spatial covariance matrix admits a Kronecker factorization in terms of much smaller matrices: 
$
\mathbf{K}_{\mathbf{X}}=\mathbf{K}_{x_1}\otimes \mathbf{K}_{x_2} \otimes
\cdots\otimes \mathbf{K}_{x_d} = \sbigotimes_{\ell=1}^d \mathbf{K}_{x_\ell}$, where
$\mathbf{K}_{x_\ell}\in\mathbb{R}^{M_\ell\times M_\ell}$ with 
$(\mathbf{K}_{x_\ell})_{ij}=k_{x_\ell}(\tilde{x}_\ell^{(i)},\tilde{x}_\ell^{(j)})$. 
Substituting this into the global Kronecker structure \eqref{kron_product-1} yields the full spatial covariance matrix
\begin{equation}
\mathbf{K}_{\mathbf{Z}}
= \mathbf{K}_{\boldsymbol{\mu}} \otimes 
\left ( 
\bigotimes_{\ell=1}^{d} \mathbf{K}_{x_\ell}
\right ) 
\otimes
\mathbf{K}_{\mathbf{T}}
.
\label{kron_product33}
\end{equation}
The parameters for the deep product kernel in this rectilinear-grid setting are denoted by $
\boldsymbol{\theta}
= \{
(\boldsymbol{\theta}^{\boldsymbol{\mu}}_{\mathrm{nn}},\boldsymbol{\theta}^{\boldsymbol{\mu}}_{\mathrm{gp}}),
(\boldsymbol{\theta}^{x_1}_{\mathrm{nn}},\boldsymbol{\theta}^{x_1}_{\mathrm{gp}}),\ldots,
(\boldsymbol{\theta}^{x_d}_{\mathrm{nn}},\boldsymbol{\theta}^{x_d}_{\mathrm{gp}}),
(\boldsymbol{\theta}^{t}_{\mathrm{nn}},\boldsymbol{\theta}^{t}_{\mathrm{gp}})
\}$, and are learned by minimizing the NLML in \eqref{NLML}.

\subsection{Scalable computation of the marginal likelihood}
\label{subsec:rect_nlml}

We now consider efficient computation of the NLML \eqref{NLML} in the rectilinear-grid setting. The dominant costs in \eqref{NLML} arise from computing the quadratic form $\mathbf{y}^T\mathbf{K}_{\mathbf{y}}^{-1}\mathbf{y}$ and computing $\log|\mathbf{K}_{\mathbf{y}}|$, where $\mathbf{K}_{\mathbf{y}}=\mathbf{K}_{\mathbf{Z}}+\sigma_n^2\mathbf{I}$ with $\mathbf{K}_{\mathbf{Z}}$ given by \eqref{kron_product33}, and $\mathbf{y}\in\mathbb{R}^{NMN_t}$ denotes the vector of training observations arranged consistently with the Kronecker structure. Rather than forming $\mathbf{K}_{\mathbf{y}}^{-1}$ explicitly, 
we exploit Kronecker algebra and tensor--matrix identities to compute the matrix-vector product 
$\mathbf{K}_{\mathbf{y}}^{-1} \mathbf{y}$.

We first compute the eigendecompositions of the parameter, spatial and temporal covariance matrices as follows:
$
\mathbf{K}_{\boldsymbol{\mu}}=\mathbf{U}_{\boldsymbol{\mu}}\mathbf{D}_{\boldsymbol{\mu}}\mathbf{U}_{\boldsymbol{\mu}}^{T}$, $\mathbf{K}_{x_\ell}=\mathbf{U}_{x_\ell}\mathbf{D}_{x_\ell}\mathbf{U}_{x_\ell}^{T}$, $\ell=1,\ldots,d$, and $\mathbf{K}_{\mathbf{T}}=\mathbf{U}_{\mathbf{T}}\mathbf{D}_{\mathbf{T}}\mathbf{U}_{\mathbf{T}}^{T}$, where $\mathbf{U}_{\boldsymbol{\mu}}$, $\mathbf{U}_{x_\ell}$, and $\mathbf{U}_{\mathbf{T}}$ are orthogonal matrices of eigenvectors, and $\mathbf{D}_{\boldsymbol{\mu}}$, $\mathbf{D}_{x_\ell}$, and $\mathbf{D}_{\mathbf{T}}$ are diagonal matrices of eigenvalues. 
Using standard Kronecker identities (see \ref{Appendix:Kron_properties}), we obtain
\begin{equation}
\begin{aligned}
\mathbf{K}_{\mathbf{y}}^{-1}
&=
\left(
\mathbf{K}_{\boldsymbol{\mu}}
\otimes
\left (
\sbigotimes_{\ell=1}^{d} \mathbf{K}_{x_\ell} \right )
\otimes
\mathbf{K}_{\mathbf{T}}
+\sigma_n^2\mathbf{I}
\right)^{-1} \\
& = \left(
\mathbf{U}_{\boldsymbol{\mu}}
\otimes
\left ( \sbigotimes_{\ell=1}^{d} \mathbf{U}_{x_\ell} \right )
\otimes
\mathbf{U}_{\mathbf{T}}
\right) \; 
\left (  \mathbf{D}_{\boldsymbol{\mu}} \otimes \left ( \sbigotimes_{\ell=1}^{d} \mathbf{D}_{x_\ell} \right )
\otimes
\mathbf{D}_{\mathbf{T}}
+\sigma_n^2\mathbf{I}
    \right )^{-1} \\
    & \qquad \qquad \times
\left ( \mathbf{U}_{\boldsymbol{\mu}}^{T} \otimes \left ( \sbigotimes_{\ell=1}^{d} \mathbf{U}_{x_\ell}^{T} \right ) \otimes \mathbf{U}_{\mathbf{T}}^{T} \right ).
\end{aligned}
\label{K_y_inv_y}
\end{equation}
The computational complexity of the step involving eigendecompositions of $\mathbf{K}_{\boldsymbol{\mu}}$, $\mathbf{K}_{\mathbf{T}}$, and 
$\{\mathbf{K}_{x_\ell}\}_{\ell=1}^{d}$ is 
$\mathcal{O} (N^3 + d\,M^{3/d} + N_t^3 )$ (assuming a balanced spatial grid with $M_\ell = M^{1/d}$),  
which is substantially smaller than the $\mathcal{O}((NMN_t)^3)$ cost of naive dense linear algebra. Although \eqref{K_y_inv_y} provides an explicit factorization, it is more efficient in practice to compute $\mathbf{K}_{\mathbf{y}}^{-1}\mathbf{y}$ without explicitly forming any large intermediate matrices. 

To illustrate, let $\mathcal{Y}$ denote the $(d+2)$-way tensor with dimensions 
$N \times M_1 \times \cdots \times M_d \times N_t$ 
obtained by reshaping $\mathbf{y}\in\mathbb{R}^{NMN_t}$. Using the standard identity relating Kronecker products to multilinear tensor products~\cite{kolda2006multilinear}, we can write
\begin{equation}
\begin{aligned}
\mathbf{K}_{\mathbf{y}}^{-1}\mathbf{y}
&=
\left ( \mathbf{U}_{\boldsymbol{\mu}} \otimes
\left ( \sbigotimes_{\ell=1}^{d} \mathbf{U}_{x_\ell} \right ) \otimes \mathbf{U}_{\mathbf{T}} \right ) \; 
\left( \mathbf{D}_{\boldsymbol{\mu}} \otimes 
\left ( \sbigotimes_{\ell=1}^d \mathbf{D}_{x_\ell}\right ) \otimes \mathbf{D}_{\mathbf{T}} + \sigma_n^2 \mathbf{I} 
\right)^{-1}
\\
&\qquad\qquad\times
\mathbf{vec} \left ( \mathcal{Y} \times_{d+2} \mathbf{U}_{\mathbf{T}}^T \times_{d+1} \mathbf{U}_{x_d}^T \ldots \times_{2} \mathbf{U}_{x_1}^T  \times_1 \mathbf{U}_{\boldsymbol{\mu}}^T \right )  \\
            &= \mathbf{vec} \left ( \left ( \left ( \mathcal{Y} \times_{d+2} \mathbf{U}_{\mathbf{T}}^T \times_{d+1} \mathbf{U}_{x_d}^T \ldots \times_{2} \mathbf{U}_{x_1}^T  \times_1 \mathbf{U}_{\boldsymbol{\mu}}^T \right ) \, \odot \, \mathcal{G}^{-1} \right ) \right . \\ 
            & \left . \qquad \qquad \qquad \times_{d+2} \mathbf{U}_{\mathbf{T}} \times_{d+1} \mathbf{U}_{x_d} \ldots \times_{2} \mathbf{U}_{x_1}  \times_1 \mathbf{U}_{\boldsymbol{\mu}} \right )
\end{aligned}
\label{Kyinv_y_tensor}
\end{equation}
where $\times_i$ denotes mode-$i$ tensor--matrix multiplication, $\mathbf{vec}(\cdot)$ vectorizes a tensor, and $\mathcal{G}$ is the $(d+2)$-way tensor formed by reshaping the diagonal entries of the matrix 
$
\mathbf{D}_{\boldsymbol{\mu}}
\otimes
( \sbigotimes_{\ell=1}^{d} \mathbf{D}_{x_\ell} )
\otimes
\mathbf{D}_{\mathbf{T}}
+\sigma_n^2\mathbf{I}$ into a $N\times M_1\times\cdots\times M_d\times N_t$ tensor.

Using \eqref{Kyinv_y_tensor}, $\mathbf{K}_{\mathbf{y}}^{-1}\mathbf{y}$ can be computed using only multiplications by the small factor matrices (without explicitly forming $\mathbf{K}_{\mathbf{y}}$ or $\mathbf{K}_{\mathbf{y}}^{-1}$), at computational complexity $\mathcal{O}( (N + dM^{\frac{1}{d}} + N_t)NMN_t )$ .
Moreover, since the eigenvalues of the matrix $\mathbf{K_y}$ coincide with the elements of the tensor $\mathcal{G}$, we can calculate the determinant of $\mathbf{K_y}$ needed in the marginal likelihood (\ref{NLML}) as 
\begin{equation}
|\mathbf{K}_{\mathbf{y}}|
=
\prod_{i_1,\ldots,i_{d+2}}
\mathcal{G}_{i_1,\ldots,i_{d+2}}.
\label{logdet_rect}
\end{equation}
The cost of forming $\mathcal{G}$ is dominated by the eigendecompositions of the Kronecker factors, i.e.,
$\mathcal{O}(N^3 + dM^{3/d} + N_t^3)$.

\subsection{Posterior mean and variance computation}
\label{subsec:rect_posterior}

We next consider how to efficiently compute the posterior mean and (diagonal) posterior variance over a spatio-temporal grid for a given test input $\boldsymbol{\mu}_*$ using Kronecker matrix algebra. We consider prediction over a Cartesian product grid of test spatial points and time points:
$
\{\mathbf{z}_*^{(i)}\}_{i=1}^{M_*N_{t*}}
=
\{ \boldsymbol{\mu}_* \} \times 
\mathbf{x}_{*,1}\times\cdots\times\mathbf{x}_{*,d}
\times
\mathbf{T}_*$, i.e., a Cartesian product over the entries of $\mathbf{x}_{*,\ell} = [x_{*,\ell}^{(1)}, \ldots, x_{*,\ell}^{({M_{*,\ell}})} ]^T \in \mathbb{R}^{{M_{*,\ell}}}$ containing the distinct coordinates along the $\ell$th spatial dimension and the vector $\mathbf{T}_* = [t_*^{(1)}, t_*^{(2)}, \ldots, t_*^{(N_{t*})} ]^T \in \mathbb{R}^{N_{t*}}$ containing the distinct test time points. 

The posterior mean and covariance over the full spatio-temporal test grid can be written as
\begin{subequations}
\begin{align}
\widehat{\mathbf{u}}(\mathbf{Z}_*)
&=
\mathbf{m}(\mathbf{Z}_*) + \mathbf{K}_{\mathbf{Z}_*,\mathbf{Z}}\mathbf{K}_{\mathbf{y}}^{-1}(\mathbf{y}-\mathbf{m}),
\label{pred_mean}
\\
\boldsymbol{\Sigma}(\mathbf{Z}_*, \mathbf{Z}_*)
&=
\mathbf{K}_{\mathbf{Z}_*,\mathbf{Z}_*}
-
\mathbf{K}_{\mathbf{Z}_*,\mathbf{Z}}\mathbf{K}_{\mathbf{y}}^{-1}\mathbf{K}_{\mathbf{Z}_*,\mathbf{Z}}^{T},
\label{pred_covar}
\end{align}
\end{subequations}
where $\widehat{\mathbf{u}}(\mathbf{Z}_*) \in \mathbb{R}^{M_*N_{t*}}$ and $\boldsymbol{\Sigma}(\mathbf{Z}_*, \mathbf{Z}_*) \in \mathbb{R}^{(M_*N_{t*}) \times (M_*N_{t*})}$ denote the posterior mean and covariance at the spatio-temporal grid points, respectively, and 
$\mathbf{Z}_* = [\mathbf{z}_*^{(1)}, \mathbf{z}_*^{(2)}, \ldots, \mathbf{z}_*^{(M_*N_{t*})} ]^T \in \mathbb{R}^{(M_*N_{t*}) \times (D+d+1)}$, and 
$M_*=\prod_{\ell=1}^{d} M_{*,\ell}$ is the total number of test spatial points.
Due to the product structured kernel, the cross-covariance  
$\mathbf{K}_{\mathbf{Z}_*,\mathbf{Z}} \in \mathbb{R}^{(M_*N_{t*}) \times (NMN_t)}$ and prior covariance matrices  
$\mathbf{K}_{\mathbf{Z}_*,\mathbf{Z}_*} \in \mathbb{R}^{(M_*N_{t*}) \times (M_*N_{t*})}$ inherit Kronecker structure:
\[
\mathbf{K}_{\mathbf{Z}_*,\mathbf{Z}}
=
\mathbf{k}_{\boldsymbol{\mu}_*}^{T}
\otimes
\left ( \sbigotimes_{\ell=1}^{d} \mathbf{K}_{x_{*,\ell},x_\ell} \right )
\otimes
\mathbf{K}_{\mathbf{T}_*,\mathbf{T}},
\]
\[
\mathbf{K}_{\mathbf{Z}_*,\mathbf{Z}_*}
=
k_{\boldsymbol{\mu}}(\boldsymbol{\mu}_*,\boldsymbol{\mu}_*)
\otimes
\left ( \sbigotimes_{\ell=1}^d \mathbf{K}_{x_{*,\ell},x_{*,\ell}} \right )
\otimes
\mathbf{K}_{\mathbf{T}_*,\mathbf{T}_*},
\]
where
$\mathbf{k}_{\boldsymbol{\mu}_*}=
[
k_{\boldsymbol{\mu}}(\boldsymbol{\mu}_*,\boldsymbol{\mu}^{(1)}),\ldots,
k_{\boldsymbol{\mu}}(\boldsymbol{\mu}_*,\boldsymbol{\mu}^{(N)})
]^T\in\mathbb{R}^{N}$, ${\mathbf{K}}_{x_{*,\ell},x_\ell}\in\mathbb{R}^{M_{*,\ell}\times M_\ell}$, 
and ${\mathbf{K}}_{\mathbf{T}_*,\mathbf{T}}\in\mathbb{R}^{N_{t*}\times N_t}$ are computed using the base kernels and feature maps defined earlier. 

While \eqref{pred_covar} defines the full posterior covariance $\boldsymbol{\Sigma}(\mathbf{Z}_*, \mathbf{Z}_*)$, storing or forming it is impractical for large $M_*N_{t*}$. We therefore focus on computing only the diagonal entries, i.e., 
\begin{equation}
\begin{split}
\mathrm{Var}(\mathbf{Z}_*)
& =
\mathrm{diag}(\boldsymbol{\Sigma}(\mathbf{Z}_*, \mathbf{Z}_*)) \\
& =
\mathrm{diag}(\mathbf{K}_{\mathbf{Z}_*,\mathbf{Z}_*})
-
\mathrm{diag} \left(
\mathbf{K}_{\mathbf{Z}_*,\mathbf{Z}}\mathbf{K}_{\mathbf{y}}^{-1}\mathbf{K}_{\mathbf{Z}_*,\mathbf{Z}}^{T}
\right).
\end{split}
\label{pred_covar_diag}
\end{equation}
The first term, $\mathrm{diag}(\mathbf{K}_{\mathbf{Z}_{*}, \mathbf{Z}_*})$, in the above equation can be efficiently computed using the identity $\mathrm{diag}(\mathbf{A} \otimes \mathbf{B}) = \mathrm{diag}(\mathbf{A}) \otimes \mathrm{diag}(\mathbf{B})$ (for square matrices $\mathbf{A}$ and $\mathbf{B}$) as follows
\begin{equation}
\mathrm{diag}(\mathbf{K}_{\mathbf{Z}_*,\mathbf{Z}_*})
=
k_{\boldsymbol{\mu}}(\boldsymbol{\mu}_*,\boldsymbol{\mu}_*)
\left ( \left ( \sbigotimes_{\ell=1}^d \mathrm{diag} \left ( \mathbf{K}_{x_{*,\ell},x_{*,\ell}} \right ) \right )
\otimes \; 
\mathrm{diag}\left ( \mathbf{K}_{\mathbf{T}_*,\mathbf{T}_*} \right ) 
\right ).
\label{var_first_term}
\end{equation}
To compute the second term in \eqref{pred_covar_diag}, we write \eqref{K_y_inv_y} in the form
$
\mathbf{K}_{\mathbf{y}}^{-1}
=
\mathbf{U}\,\mathrm{diag}(\boldsymbol{\lambda})\,\mathbf{U}^{T}$, where 
$\mathbf{U}  =
\mathbf{U}_{\boldsymbol{\mu}}
\otimes 
\left ( \sbigotimes_{\ell=1}^d \mathbf{U}_{x_\ell} \right )
\otimes
\mathbf{U}_{\mathbf{T}}$ and
$$\mathrm{diag}(\boldsymbol{\lambda})
 =
\left(
\mathbf{D}_{\boldsymbol{\mu}}
\otimes
\left ( \sbigotimes_{\ell=1}^d \mathbf{D}_{x_\ell} \right )
\otimes
\mathbf{D}_{\mathbf{T}}
+\sigma_n^2\mathbf{I}
\right)^{-1}.$$ 
Using the identity $\mathrm{diag}(\mathbf{A}\,\mathrm{diag}(\mathbf{d})\,\mathbf{A}^{T})=(\mathbf{A}\odot\mathbf{A})\mathbf{d}$, where $\odot$ denotes the Hadamard product, we obtain
\begin{equation}
\begin{aligned}
\mathrm{diag} \left (
\mathbf{K}_{\mathbf{Z}_*,\mathbf{Z}}\mathbf{K}_{\mathbf{y}}^{-1}\mathbf{K}_{\mathbf{Z}_*,\mathbf{Z}}^{T}
\right )
&=
\left(
\left ( \mathbf{K}_{\mathbf{Z}_*,\mathbf{Z}}\mathbf{U} \right )
\odot
\left ( \mathbf{K}_{\mathbf{Z}_*,\mathbf{Z}}\mathbf{U} \right )
\right) \, \boldsymbol{\lambda}, 
\end{aligned}
\label{diag_second_term_general}
\end{equation}
where  
$$
\mathbf{K}_{\mathbf{Z}_*,\mathbf{Z}}\mathbf{U}
=
\left ( \mathbf{k}_{\boldsymbol{\mu}_*}^{T}\mathbf{U}_{\boldsymbol{\mu}} \right )
\otimes
\left ( 
    \sbigotimes_{\ell=1}^d \left ( \mathbf{K}_{x_{*,\ell},x_\ell}
    \mathbf{U}_{x_\ell} \right )
\right )
\otimes
\left ( \mathbf{K}_{\mathbf{T}_*,\mathbf{T}}\mathbf{U}_{\mathbf{T}} \right ).
$$
Using the identity $(\mathbf{A}\otimes\mathbf{B})\odot(\mathbf{C}\otimes\mathbf{D})=(\mathbf{A}\odot\mathbf{C})\otimes(\mathbf{B}\odot\mathbf{D})$, \eqref{diag_second_term_general} becomes
\begin{equation}
\begin{aligned}
\mathrm{diag} \left (
\mathbf{K}_{\mathbf{Z}_*,\mathbf{Z}}\mathbf{K}_{\mathbf{y}}^{-1}\mathbf{K}_{\mathbf{Z}_*,\mathbf{Z}}^{T}
\right )
& =
\Bigl[
(\mathbf{k}_{\boldsymbol{\mu}_*}^{T}\mathbf{U}_{\boldsymbol{\mu}}\odot \mathbf{k}_{\boldsymbol{\mu}_*}^{T}\mathbf{U}_{\boldsymbol{\mu}}) \\
& \quad \otimes
\left(
\sbigotimes_{\ell=1}^d \left ( \mathbf{K}_{x_{*,\ell},x_\ell}\mathbf{U}_{x_\ell}\odot \mathbf{K}_{x_{*,\ell},x_\ell}\mathbf{U}_{x_\ell} \right )
\right) \\
& \quad \otimes
(\mathbf{K}_{\mathbf{T}_*,\mathbf{T}}\mathbf{U}_{\mathbf{T}}\odot \mathbf{K}_{\mathbf{T}_*,\mathbf{T}}\mathbf{U}_{\mathbf{T}})
\Bigr] \, \boldsymbol{\lambda}.
\label{var_second_term}
\end{aligned}
\end{equation}
Using \eqref{var_first_term} and \eqref{var_second_term}, we can compute the diagonal entries of the posterior predictive variance $\mathrm{Var}(\mathbf{Z}_*)$ in \eqref{pred_covar_diag}. 

Leveraging Kronecker algebra and structured tensor--matrix multiplications, we can compute the predictive mean in just $\mathcal{O}( (N + dM^{\frac{1}{d}} + N_t)NMN_t )$ operations and $\mathcal{O}( NN_t M + d  M^{\frac{2}{d}})$ storage when the test spatial/time grid equals the training grid. Similarly, the posterior predictive variance can be computed efficiently with  similar computational cost. In typical applications involving high-resolution PDE solutions, the number of spatial points $M$ is significantly larger than the number of parameter points $N$ and the number of time points $N_t$, i.e., $M\gg N$ and $M\gg N_t$.
This results in the computational and storage costs scaling nearly linearly with the number of spatial points $M$ (up to dimension-dependent factors), enabling exact GP training and inference for high-resolution spatio-temporal datasets.

\section{Gaussian process modeling of parametrized spatio-temporal fields on general spatial grids}
\label{section:general-grids}

We now address the case in which the scalar field $u$ is sampled on a general (unstructured) spatial grid. In this setting, the Kronecker factorization leveraged in Section~\ref{section:Struct-GPR} is no longer applicable since the spatial locations do not form a rectilinear Cartesian product grid. A natural alternative is to map the data to a rectilinear grid and then apply the rectilinear-grid method. However, for PDEs posed on general domains (e.g., Figure~\ref{fig:ref_domain_map}), such a mapping typically introduces \emph{gaps}: spatial locations on the rectilinear grid where the field is undefined (for instance, points lying inside a solid body such as the airfoil interior). To overcome this difficulty, we propose a \emph{gappy-grid extension} of the rectilinear-grid method that enables scalable training and inference while accounting for undefined spatial locations.

\subsection{Gappy rectilinear embedding and problem setup}
\label{subsec:gappy_setup}

In the proposed approach, the field $u$ is embedded onto a rectilinear grid with $M$ points denoted by 
$\mathbf{X}=\{\mathbf{x}^{(i)}\}_{i=1}^{M}$. Specifically, we introduce a fixed rectilinear \emph{background} spatial  grid $\mathbf{X}$ that covers the computational domain of interest (e.g., a bounding box in physical space). For each parameter--time snapshot, the field values are represented on $\mathbf{X}$ by transferring the available solution data from the original unstructured mesh to the background grid (e.g., by interpolation or projection). For complex domains this embedding produces gaps, i.e., locations where $u$ is undefined. Grid points in $\mathbf{X}$ that do not belong to the physical domain (or where the transfer is undefined) are treated as missing entries, which we refer to as \emph{gaps}. 

We partition the rectilinear grid into regular and gappy locations as 
$\mathbf{X} = \mathbf{X}_{r}\,\cup\,\mathbf{X}_{g}$ with 
$\mathbf{X}_{r}\cap\mathbf{X}_{g}=\emptyset$, 
where $\mathbf{X}_{r}=\{\mathbf{x}_{r}^{(i)}\}_{i=1}^{N_r}$ are \emph{regular} locations with known values, 
$\mathbf{X}_{g}=\{\mathbf{x}_{g}^{(i)}\}_{i=1}^{N_g}$ are \emph{gappy} locations with undefined values, and $M=N_r+N_g$.

To learn the scalar field $u$, consider a GP prior
$u(\mathbf{z})\sim\mathcal{GP}(m(\mathbf{z}),k(\mathbf{z},\mathbf{z}'))$ \cite{williams2006gaussian},
where $\mathbf{z}=[\boldsymbol{\mu},\mathbf{x},t]^T\in\mathbb{R}^{D+d+1}$, and observations satisfy
$y^{(i)} = u(\mathbf{z}^{(i)}) + \epsilon^{(i)}$ with $\epsilon^{(i)}\sim\mathcal{N}(0,\sigma_n^2)$. In what follows, we will set the prior mean function to zero without loss of generality, i.e., $m(\mathbf{z})\equiv0$. 

If only the regular observations are used to estimate the GP hyperparameters, the training set corresponds to inputs
$\mathbf{Z}_{r}=[\mathbf{z}_{r}^{(1)}, \mathbf{z}_{r}^{(2)}, \ldots, \mathbf{z}_{r}^{(N N_r N_t)}]^T$ with $\mathbf{z}_{r}=[\boldsymbol{\mu},\mathbf{x}_{r},t]^T$,
and observation vector
$\mathbf{y}_{r}\in\mathbb{R}^{N N_r N_t}$. The NLML for this gappy-grid GP model is given by
\begin{equation}
\begin{aligned}
\mathrm{NLML}(\boldsymbol{\theta})
&=
\frac{1}{2} 
 \mathbf{y}_r^{T} \left ( \mathbf{K}_{\mathbf{Z}_r}+\sigma_n^2\mathbf{I} \right )^{-1} \mathbf{y}_r 
 + \frac{1}{2}\log|\mathbf{K}_{\mathbf{Z}_r}+\sigma_n^2\mathbf{I}| \\
& \quad +
\frac{N N_r N_t}{2}\log(2\pi),
\end{aligned}
\label{NLML_gappy}
\end{equation}
where $\mathbf{K}_{\mathbf{Z}_r} =
\mathbf{K}_{\boldsymbol{\mu}}
\otimes
\mathbf{K}_{\mathbf{X}_r}
\otimes
\mathbf{K}_{\mathbf{T}} \in \mathbb{R}^{(N N_r N_t)\times(N N_r N_t)}$ denotes the prior covariance matrix evaluated at the regular training inputs $\mathbf{Z}_r$. The posterior mean at a test input $\mathbf{z}_*$ is given by
\begin{equation}
\widehat{u}(\mathbf{z}_*)= \mathbf{k}(\mathbf{z}_*)^T \left ( \mathbf{K}_{\mathbf{Z}_r} + \sigma_n^2 \mathbf{I} \right )^{-1} \mathbf{y}_r,
\label{pred_mean_gappy}
\end{equation}
where $\mathbf{k}(\mathbf{z}_*) = [ k(\mathbf{z}_*,\mathbf{z}_r^{(1)}), \ldots, k(\mathbf{z}_*,\mathbf{z}_r^{(N N_r N_t)}) ]^T \in \mathbb{R}^{N N_r N_t}$.

Since the spatial covariance matrix over the regular points 
$\mathbf{K}_{\mathbf{X}_r}$ does not have a Kronecker product structure, the computational cost of training this GP model scales as $\mathcal{O}((N^3 + N_r^3 + N_t^3) + (N + N_r + N_t) N N_r N_t)$, which is prohibitive for large spatial grids. We next describe a scalable approach to overcome this challenge.

\subsection{Scalable training and inference}
\label{subsec:gappy_pseudovalues}

To address the aforementioned computational challenge, we require efficient algorithms for the following computations  encountered during GP training and inference: (1)~repeated application of $(\mathbf{K}_{\mathbf{Z}_r}+\sigma_n^2\mathbf{I})^{-1}$ to vectors (e.g., to compute $(\mathbf{K}_{\mathbf{Z}_r}+\sigma_n^2\mathbf{I})^{-1}\mathbf{y}_r$), (2)~computation of the quadratic form $\mathbf{y}_r^T(\mathbf{K}_{\mathbf{Z}_r}+\sigma_n^2\mathbf{I})^{-1}\mathbf{y}_r$ in the NLML, and (3)~approximation of the log-determinant $\log|\mathbf{K}_{\mathbf{Z}_r}+\sigma_n^2\mathbf{I}|$. 
We address the computational challenges arising from the lack of Kronecker product structure by adapting an idea proposed for GP regression of gappy spatial datasets~\cite{evans2018exploiting} to the parametrized spatio-temporal setting. The key idea is to replace operations on the unstructured matrix $\mathbf{K}_{\mathbf{Z}_r}$ by operations on a \emph{larger} linear system whose covariance matrix preserves a Kronecker product structure and therefore admits fast linear algebra operations. 

Let $\mathbf{Z}=[\mathbf{z}^{(1)}, \mathbf{z}^{(2)}, \ldots, \mathbf{z}^{(N M N_t)}]^T$ denote the full set of training inputs after embedding the field $u$ that was originally provided over a general unstructured grid to the rectilinear grid $\mathbf{X}$. Let $\mathbf{T}=[t^{(1)}, t^{(2)}, \ldots, t^{(N_t)}]^T \in \mathbb{R}^{N_t}$ denote the full set of time points, ordered consistently with the Kronecker product structure, and let $\mathbf{y}\in\mathbb{R}^{N M N_t}$ denote the full  observation vector, formed by combining the known values at regular locations $\mathbf{y}_r$ and unknown $\mathbf{y}_g\in\mathbb{R}^{N N_g N_t}$ at the gappy locations.
We partition $\mathbf{X}$ into regular and gappy subsets via the spatial selection matrices $\widetilde{\mathbf{W}}\in\mathbb{R}^{N_r\times M}$ and $\widetilde{\mathbf{V}}\in\mathbb{R}^{N_g\times M}$, and lift them to the parametrized spatio-temporal training set using
${\mathbf{W}}=\mathbf{I}_{N}\otimes \widetilde{\mathbf{W}}\otimes \mathbf{I}_{N_t}$ and 
${\mathbf{V}}=\mathbf{I}_{N}\otimes \widetilde{\mathbf{V}}\otimes \mathbf{I}_{N_t}$, 
so that $\mathbf{y}_r={\mathbf{W}}\mathbf{y}$ and $\mathbf{y}_g={\mathbf{V}}\mathbf{y}$. The selection matrices ${\mathbf{W}}$ and ${\mathbf{V}}$ are sparse matrices with each row containing a single entry equal to one, and they satisfy ${\mathbf{W}}{\mathbf{W}}^{T}=\mathbf{I}$, ${\mathbf{V}}{\mathbf{V}}^{T}=\mathbf{I}$, and ${\mathbf{W}}{\mathbf{V}}^{T}=\mathbf{0}$.

Since $\mathbf{X}$ is rectilinear, the full covariance matrix over $\mathbf{Z}$ takes the Kronecker product form
$\mathbf{K}_{\mathbf{Z}}
=
\mathbf{K}_{\boldsymbol{\mu}}
\otimes
\left(
\bigotimes_{\ell=1}^{d}\mathbf{K}_{x_\ell}
\right)
\otimes
\mathbf{K}_{\mathbf{T}} \in \mathbb{R}^{(N M N_t) \times (N M N_t)} $. We adapt the theoretical analysis of~\cite{evans2018exploiting} to formalize how computations involving linear solves with the unstructured matrix $\mathbf{K}_{\mathbf{Z}_r} + \sigma_n^2 \mathbf{I}$ can be exactly replaced by a larger system involving the Kronecker-structured matrix $\mathbf{K}_{\mathbf{Z}} + \sigma_n^2 \mathbf{I}$ using appropriately defined pseudovalues at the gaps in the spatial grid. 

\begin{lemma}
\label{lem:fill_gaps_equivalence} 
The solution of 
$(\mathbf{K}_{\mathbf{Z}_r}+\sigma_n^2\mathbf{I})\boldsymbol{\alpha}_r=\mathbf{y}_r$ can be written as $\boldsymbol{\alpha}_r := {\mathbf{W}}\boldsymbol{\alpha}$, where $\boldsymbol{\alpha}\in\mathbb{R}^{N M N_t}$ is the solution of the Kronecker product structured system of linear algebraic equations 
\begin{equation}
\label{eq:full_structured_system}
(\mathbf{K}_{\mathbf{Z}}+\sigma_n^2\mathbf{I})\boldsymbol{\alpha}=\mathbf{y}, ~\text{with}~~ \mathbf{y}={\mathbf{W}}^T\mathbf{y}_r+{\mathbf{V}}^T\mathbf{y}_g, 
\end{equation}
where $\mathbf{y}_r\in\mathbb{R}^{NN_rN_t}$ denotes the known observations at regular locations and 
$\mathbf{y}_g\in\mathbb{R}^{NN_gN_t}$ denotes the pseudovalues at the gaps that are uniquely determined by the linear algebraic system of equations:
\begin{equation}
\left (
{\mathbf{V}}\, \left ( \mathbf{K}_{\mathbf{Z}} + \sigma_n^2 \mathbf{I} \right )^{-1} {\mathbf{V}}^{T} \right ) \, \mathbf{y}_{g}
=
-{\mathbf{V}}\,\left ( \mathbf{K}_{\mathbf{Z}} + \sigma_n^2 \mathbf{I} \right )^{-1} {\mathbf{W}}^{T}\mathbf{y}_{r}.
\label{eq:fg_system_lemma}
\end{equation}
\end{lemma}

\begin{proof}
We first establish the existence and uniqueness of $\mathbf{y}_g$ in \eqref{eq:fg_system_lemma}. Since $\mathbf{K}_{\mathbf{Z}}+\sigma_n^2\mathbf{I}$ is symmetric positive definite (SPD) by construction for $\sigma_n^2>0$, its inverse exists and is SPD. Furthermore, since the selection matrix ${\mathbf{V}}$ has full row rank, we have ${\mathbf{V}}^T\mathbf{q}\neq \mathbf{0}$ whenever $\mathbf{q}\neq \mathbf{0}$. It then follows that for any nonzero vector $\mathbf{q}\in\mathbb{R}^{NN_gN_t}$ 
$$
\mathbf{q}^T\bigl({\mathbf{V}}(\mathbf{K}_{\mathbf{Z}}+\sigma_n^2\mathbf{I})^{-1}{\mathbf{V}}^T\bigr)\mathbf{q}
=
(\mathbf{V}^T\mathbf{q})^T(\mathbf{K}_{\mathbf{Z}}+\sigma_n^2\mathbf{I})^{-1}(\mathbf{V}^T\mathbf{q}) \, >\, 0.$$ 
We therefore conclude that 
${\mathbf{V}}(\mathbf{K}_{\mathbf{Z}}+\sigma_n^2\mathbf{I})^{-1}{\mathbf{V}}^T$ is SPD and hence invertible, leading to a unique solution $\mathbf{y}_g$ for \eqref{eq:fg_system_lemma}.

Let $\mathbf{y}$ denote the full observation vector and let $\boldsymbol{\alpha} = (\mathbf{K}_{\mathbf{Z}}+\sigma_n^2\mathbf{I})^{-1}\mathbf{y}$ denote the solution of  \eqref{eq:full_structured_system}. Since the selection matrices 
${\mathbf{W}}$ and ${\mathbf{V}}$ select disjoint index sets that partition the full grid (i.e., 
$\mathbf{W}\mathbf{V}^T = \mathbf{0}$, $\mathbf{W}\mathbf{W}^T = \mathbf{I}$, $\mathbf{V}\mathbf{V}^T = \mathbf{I}$), we have $\mathbf{y}={\mathbf{W}}^T\mathbf{y}_r+{\mathbf{V}}^T\mathbf{y}_g$. The coefficients at the gappy locations can be written as
\begin{equation}
\begin{split}
{\mathbf{V}}\boldsymbol{\alpha} & ={\mathbf{V}}(\mathbf{K}_{\mathbf{Z}}+\sigma_n^2\mathbf{I})^{-1}\mathbf{y} = {\mathbf{V}}(\mathbf{K}_{\mathbf{Z}}+\sigma_n^2\mathbf{I})^{-1}({\mathbf{W}}^T\mathbf{y}_r+{\mathbf{V}}^T\mathbf{y}_g) \\
& =
{\mathbf{V}}(\mathbf{K}_{\mathbf{Z}}+\sigma_n^2\mathbf{I})^{-1}{\mathbf{W}}^{T}\mathbf{y}_{r}
+
{\mathbf{V}}(\mathbf{K}_{\mathbf{Z}}+\sigma_n^2\mathbf{I})^{-1}{\mathbf{V}}^{T}\mathbf{y}_{g} \\
& = \mathbf{0} .
\end{split}
\label{eq:alpha_g_zero}
\end{equation}
In the last step, we use \eqref{eq:fg_system_lemma} to substitute for the second term on the right-hand side. This shows that the coefficients at the gappy locations, as derived from \eqref{eq:full_structured_system}, are zero, when the pseudovalues are chosen according to \eqref{eq:fg_system_lemma}, i.e., $\boldsymbol{\alpha}_g:={\mathbf{V}}\boldsymbol{\alpha}=\mathbf{0}$.

Next, we note that the full structured system \eqref{eq:full_structured_system} can be rewritten in partitioned form using the regular and gappy index sets induced by the selection matrices $({\mathbf{W}},{\mathbf{V}})$ as follows:
\begin{equation}
\left (\begin{array}{cc}
{\mathbf{W}} (\mathbf{K}_{\mathbf{Z}}+\sigma_n^2\mathbf{I}) \mathbf{W}^T & {\mathbf{W}} (\mathbf{K}_{\mathbf{Z}}+\sigma_n^2\mathbf{I}) \mathbf{V}^T \\
{\mathbf{V}} (\mathbf{K}_{\mathbf{Z}}+\sigma_n^2\mathbf{I}) \mathbf{W}^T & {\mathbf{V}} (\mathbf{K}_{\mathbf{Z}}+\sigma_n^2\mathbf{I}) \mathbf{V}^T
\end{array}\right )
\left (\begin{array}{c}
\boldsymbol{\alpha}_r \\
\boldsymbol{\alpha}_g
\end{array}\right )
=
\left (\begin{array}{c}
\mathbf{y}_r \\
\mathbf{y}_g
\end{array}\right ).
\label{eq:block_structured_system}
\end{equation}
Since $\boldsymbol{\alpha}_g=\mathbf{0}$ by \eqref{eq:alpha_g_zero}, the first block row can be written as
\begin{equation}
{\mathbf{W}} \left (\mathbf{K}_{\mathbf{Z}}+\sigma_n^2\mathbf{I} \right ) \mathbf{W}^T \boldsymbol{\alpha}_r  =
\mathbf{y}_r.
\label{eq:block_row}
\end{equation}
Noting that ${\mathbf{W}}(\mathbf{K}_{\mathbf{Z}}+\sigma_n^2\mathbf{I}){\mathbf{W}}^T
=
{\mathbf{W}}\mathbf{K}_{\mathbf{Z}}\mathbf{W}^T+\sigma_n^2\mathbf{I}
=
\mathbf{K}_{\mathbf{Z}_r}+\sigma_n^2\mathbf{I}$, it follows that the solution of $\boldsymbol{\alpha}_r$ provided by \eqref{eq:full_structured_system} coincides with the 
solution of the original linear algebraic system with the covariance matrix defined using only the regular points. This completes the proof.
\end{proof}

Lemma~\ref{lem:fill_gaps_equivalence} provides the key justification for the gappy-grid approach: by selecting pseudovalues $\mathbf{y}_g$ through \eqref{eq:fg_system_lemma}, we can exactly replace the original unstructured system with the Kronecker-product structured  full-grid system \eqref{eq:full_structured_system} without altering the resulting coefficients on the regular points. 
In practice, we compute $\mathbf{y}_g$ using the conjugate gradient (CG) method with Kronecker-accelerated MVMs, leading to the full reconstruction $\mathbf{y}={\mathbf{W}}^T\mathbf{y}_r+{\mathbf{V}}^T\mathbf{y}_g$. The computational cost per CG iteration is $\mathcal{O}( (N + dM^{\frac{1}{d}} + N_t)NMN_t )$ using the approach in Section~\ref{section:Struct-GPR}.

Lemma~\ref{lem:fill_gaps_equivalence} enables exact evaluation of the quadratic form in the NLML defined in  \eqref{NLML_gappy} using the solution of the Kronecker structured full-grid system. Indeed, letting $\boldsymbol{\alpha}_r=(\mathbf{K}_{\mathbf{Z}_r}+\sigma_n^2\mathbf{I})^{-1}\mathbf{y}_r$ and $\boldsymbol{\alpha}=(\mathbf{K}_{\mathbf{Z}}+\sigma_n^2\mathbf{I})^{-1}\mathbf{y}$, Lemma~\ref{lem:fill_gaps_equivalence} gives $\boldsymbol{\alpha}_r={\mathbf{W}}\boldsymbol{\alpha}$. Therefore,
\[
\mathbf{y}_r^T(\mathbf{K}_{\mathbf{Z}_r}+\sigma_n^2\mathbf{I})^{-1}\mathbf{y}_r
=
\mathbf{y}_r^T\boldsymbol{\alpha}_r
=
\mathbf{y}_r^T{\mathbf{W}}\boldsymbol{\alpha}
=
({\mathbf{W}}^T\mathbf{y}_r)^T\boldsymbol{\alpha}
=
\mathbf{y}^T(\mathbf{K}_{\mathbf{Z}}+\sigma_n^2\mathbf{I})^{-1}\mathbf{y},
\]
where we used $\mathbf{y}={\mathbf{W}}^T\mathbf{y}_r + \mathbf{V}^T \mathbf{y}_g$ and $\mathbf{V} \boldsymbol{\alpha} = \mathbf{0}$. Consequently, the quadratic form can be computed over the full rectilinear grid using Kronecker-accelerated solves at computational cost $\mathcal{O}((N + dM^{\frac{1}{d}} + N_t)NMN_t)$.

In order to estimate the GP hyperparameters by minimizing the NLML in \eqref{NLML_gappy}, we also need to compute the log-determinant term $\log|\mathbf{K}_{\mathbf{Z}_r}+\sigma_n^2\mathbf{I}|$. We note that $\mathbf{K}_{\mathbf{Z}_r}$ is (up to a permutation of indices) a principal submatrix of $\mathbf{K}_{\mathbf{Z}}$ obtained by restricting the full-grid covariance matrix to the regular locations. Consequently, the eigenvalues of $\mathbf{K}_{\mathbf{Z}_r}$ interlace those of $\mathbf{K}_{\mathbf{Z}}$~\cite{hwang2004cauchy}, i.e., 
$\lambda_i \, \ge\, \tilde\lambda_i \,\ge\, \lambda_{i+n_g}$, 
for $i=1,\ldots,n_r$, where $\lambda_1\ge\cdots\ge\lambda_n$ and $\tilde\lambda_1\ge\cdots\ge\tilde\lambda_{n_r}$ denote  the eigenvalues of $\mathbf{K}_{\mathbf{Z}}$ and $\mathbf{K}_{\mathbf{Z}_r}$, respectively, $n:=NMN_t$, $n_r:=NN_rN_t$, and $n_g:=NN_gN_t$. Since $\log(\cdot)$ is strictly monotonic, it follows that 
\begin{equation}
\sum_{i=1}^{n_r}\log(\lambda_i  +\sigma_n^2)
\;\ge\;
\log | (\mathbf K_{\mathbf{Z}_r} + \sigma_n^2 \mathbf{I}) | 
\;\ge\;
\sum_{i=1}^{n_r}\log(\lambda_{i+n_g} + \sigma_n^2 ),
\label{eq:logdet_bounds}
\end{equation}
which suggests that the log-determinant of $\mathbf{K}_{\mathbf{Z}_r}$ can be approximated using the leading eigenvalues of the structured full-grid matrix $\mathbf{K}_{\mathbf{Z}}$ using a  correction factor to account for the reduced number of observations relative to the full grid. 
Motivated by this property, we employ a Nystr\"{o}m-type correction~\cite{evans2018exploiting,williams2000using,wilson2014fast} of the form
\begin{equation}
\log\left|\mathbf{K}_{\mathbf{Z}_r} + \sigma_n^2 \mathbf{I}\right|
\approx
\sum_{i=1}^{N N_r N_t}
\log\!\left(\frac{N N_r N_t}{N M N_t}\lambda_i + \sigma_n^2\right),
\label{log_det_fill_gaps}
\end{equation}
where $\lambda_i$ are the largest ${NN_{r}N_t}$ eigenvalues of $(\mathbf{K}_{\mathbf{Z}} + \sigma_n^2 \mathbf{I})$.  
Eigenvalues of $\mathbf{K}_{\mathbf{Z}}$ can be computed efficiently using Kronecker algebra in $\mathcal{O}(N^3 + dM^{3/d} + N_t^3)$ operations. Subsequently, the parameters of the deep product kernel can be jointly learned by minimizing the NLML.

In summary, the gappy-grid extension for general unstructured spatial grids involves (i)~embedding the field onto a rectilinear grid with gaps, (ii)~estimating pseudovalues at gappy locations by solving \eqref{eq:fg_system_lemma} using the CG method, and (iii)~minimizing the NLML using computations based on the full observation vector $\mathbf{y}$ formed by combining known values and pseudovalues. We next show how the posterior mean and variance at test points can be efficiently computed for the gappy-grid approach.

\subsection{Efficient computation of the posterior mean}
\label{subsec:posterior_mean_variance}

We now examine how to efficiently compute the posterior mean  over a spatio-temporal grid for a given test parameter $\boldsymbol{\mu}_{*}$ using Kronecker matrix algebra. We consider prediction over a rectilinear spatial grid with $M_*$ points denoted by $\mathbf{X}_* = \{\mathbf{x}_*^{(i)}\}_{i=1}^{M_*}$, leading to the set of test spatio-temporal grid points 
$
\{\mathbf{z}_*^{(i)}\}_{i=1}^{M_*N_{t*}}
=
\{ \boldsymbol{\mu}_* \} \times 
\mathbf{x}_{*,1}\times\cdots\times\mathbf{x}_{*,d}
\times
\mathbf{T}_*$, i.e., a Cartesian product over the entries of $\mathbf{x}_{*,\ell} = [x_{*,\ell}^{(1)}, \ldots, x_{*,\ell}^{({M_{*,\ell}})} ]^T \in \mathbb{R}^{{M_{*,\ell}}}$ containing the distinct coordinates along the $\ell$th spatial dimension and the vector $\mathbf{T}_* = [t_*^{(1)}, t_*^{(2)}, \ldots, t_*^{(N_{t*})} ]^T \in \mathbb{R}^{N_{t*}}$ containing the distinct test time points.  We will compactly denote the full test spatio-temporal grid as 
$\mathbf{Z}_{*} = [\mathbf{z}_*^{(1)}, \mathbf{z}_*^{(2)}, \ldots, \mathbf{z}_*^{(M_* N_{t*})}]^T \in \mathbb{R}^{M_*N_{t*} \times (D+d+1)}$. 

Similar to how we partitioned the training spatial grid $\mathbf{X}$ into regular and gappy subsets, we partition the test spatial grid $\mathbf{X}_*$ into a regular subset $\mathbf{X}_{r*}$ with $N_{r*}$ points and a gappy subset $\mathbf{X}_{g*}$ with $N_{g*}$ points via the spatial selection matrices $\widetilde{\mathbf{W}}_* \in \mathbb{R}^{N_{r*} \times M_*}$ and $\widetilde{\mathbf{V}}_* \in \mathbb{R}^{N_{g*} \times M_*}$, respectively. Our goal is to compute the posterior mean 
at test points $\mathbf{Z}_{r*} = \{ \boldsymbol{\mu}_* \} \times \mathbf{X}_{r*} \times \mathbf{T}_*$. 

The posterior mean over the full test spatio-temporal grid $\mathbf{Z}_{*}$, considering only the regular observations $\mathbf{y}_r$, takes the form  $\widehat{\mathbf{u}}(\mathbf{Z}_{*})
= \mathbf{K}_{\mathbf{Z}_{*},\mathbf{Z}_r} \boldsymbol{\alpha}_r \in \mathbb{R}^{M_* N_{t*}}$, where $\boldsymbol{\alpha}_r=(\mathbf{K}_{\mathbf{Z}_r}+\sigma_n^2\mathbf{I})^{-1}\mathbf{y}_r$. Since $\mathbf{K}_{\mathbf{Z}_r}$ does not have a Kronecker product structure, we will use Lemma~\ref{lem:fill_gaps_equivalence} to obtain an expression for the posterior mean that can be efficiently computed using Kronecker matrix algebra.
We begin by noting that the cross-covariance matrix between the test and regular training points, 
$\mathbf{K}_{\mathbf{Z}_{*},\mathbf{Z}_r}$, can be expressed in terms of the full cross-covariance matrix $\mathbf{K}_{\mathbf{Z}_{*},\mathbf{Z}}$ 
(that includes both regular and gappy training points) 
and the selection matrix $\mathbf{W}$ as 
$ \mathbf{K}_{\mathbf{Z}_{*},\mathbf{Z}_r} = \mathbf{K}_{\mathbf{Z}_{*},\mathbf{Z}} \mathbf{W}^T$. Moreover, the cross-covariance matrix $\mathbf{K}_{\mathbf{Z}_*,\mathbf{Z}} $ $ \in \mathbb{R}^{(M_* N_{t*}) \times (N M N_t)}$ admits the Kronecker product representation 
$
\mathbf{K}_{\mathbf{Z}_*,\mathbf{Z}}
=
\mathbf{k}_{\boldsymbol{\mu}_*}^{T}
\otimes
 ( \sbigotimes_{\ell=1}^{d} \mathbf{K}_{x_{*,\ell},x_\ell} )
\otimes
\mathbf{K}_{\mathbf{T}_*,\mathbf{T}}$, 
where
$\mathbf{k}_{\boldsymbol{\mu}_*}=
[
k_{\boldsymbol{\mu}}(\boldsymbol{\mu}_*,\boldsymbol{\mu}^{(1)}),\ldots,
k_{\boldsymbol{\mu}}(\boldsymbol{\mu}_*,\boldsymbol{\mu}^{(N)})
]^T\in\mathbb{R}^{N}$. 
Using Lemma~\ref{lem:fill_gaps_equivalence} and noting that by definition $\mathbf{W} \mathbf{W}^T = \mathbf{I}$, we can therefore express the posterior mean at the full test spatio-temporal grid $\mathbf{Z}_*$ as
\begin{equation}
    \widehat{\mathbf{u}}(\mathbf{Z}_{*})
    = \mathbf{K}_{\mathbf{Z}_{*},\mathbf{Z}} \mathbf{W}^T {\mathbf{W}}\boldsymbol{\alpha} = 
\left ( \mathbf{k}_{\boldsymbol{\mu}_*}^{T}
\otimes
\left ( \sbigotimes_{\ell=1}^{d} \mathbf{K}_{x_{*,\ell},x_\ell} \right )
\otimes
\mathbf{K}_{\mathbf{T}_*,\mathbf{T}} \right ) 
    \boldsymbol{\alpha},
\label{eq:mean_full_post}
\end{equation}
where $\boldsymbol{\alpha}=(\mathbf{K}_{\mathbf{Z}}+\sigma_n^2\mathbf{I})^{-1}\mathbf{y}$, and $\mathbf{y}={\mathbf{W}}^T\mathbf{y}_r+{\mathbf{V}}^T \mathbf{y}_g$ is the reconstructed full observation vector, $\mathbf{y}_g$ is computed by solving \eqref{eq:fg_system_lemma} using the CG method. 
Finally, we extract the posterior mean at the regular spatial points $\mathbf{Z}_{r*}$ from the full posterior mean $\widehat{\mathbf{u}}(\mathbf{Z}_{*})$ as follows:
\begin{equation}
\widehat{\mathbf{u}}(\mathbf{Z}_{r*})
= (\widetilde{\mathbf{W}}_* \otimes \mathbf{I}_{N_{t*}}) \widehat{\mathbf{u}}(\mathbf{Z}_{*}),
\label{eq:mean_r_post}
\end{equation}
where $\widehat{\mathbf{u}}(\mathbf{Z}_{r*}) \in \mathbb{R}^{N_{r*} N_{t*}}$, $\mathbf{Z}_{r*}=[\mathbf{z}_{r*}^{(1)}, \mathbf{z}_{r*}^{(2)}, \ldots, \mathbf{z}_{r*}^{(N_r N_{t*})}]^T \in \mathbb{R}^{N_{r*} N_{t*} \times (D+d+1)}$ with $\mathbf{z}_{r*}=[\boldsymbol{\mu}_*,\mathbf{x}_{*r},t_*]^T$, and $\widetilde{\mathbf{W}}_* \in \mathbb{R}^{N_{r*} \times M_*}$ is the spatial selection matrix for the regular spatial points on the test spatial grid.

It follows from \eqref{eq:mean_full_post} and \eqref{eq:mean_r_post}
that by leveraging Kronecker algebra and structured tensor--matrix multiplications, we can compute the predictive mean in just $\mathcal{O}( (N + dM^{\frac{1}{d}} + N_t)NMN_t )$ operations and $\mathcal{O}( NN_t M + d  M^{\frac{2}{d}})$ storage when the test spatial/time grid equals the training grid. We note that the posterior mean computed using \eqref{eq:mean_r_post} is exact and consistent with \eqref{pred_mean_gappy} on the regular training inputs by Lemma~\ref{lem:fill_gaps_equivalence}.

\subsection{Efficient computation of the posterior variance}

We next consider efficient computation of the posterior variance at test points using the gappy-grid approach. Our goal is to compute the posterior variance at a test point $\mathbf{z}_*$, considering only the regular observations $\mathbf{y}_r$. Recall that the exact posterior variance at a test input
$\mathbf{z}_*=[\boldsymbol{\mu}_*,\mathbf{x}_*,t_*]^T$ conditioned on the regular training inputs
$\mathbf{Z}_r$ takes the form $\mathrm{Var}(\mathbf{z}_*) = k(\mathbf{z}_*,\mathbf{z}_*) - {\mathbf k}_r(\mathbf{z}_*)^T \left(\mathbf{K}_{\mathbf{Z}_r}+\sigma_n^2 \mathbf{I}\right)^{-1}{\mathbf k}_r(\mathbf{z}_*)$, where ${\mathbf k}_r(\mathbf{z}_*) := \mathbf{K}_{\mathbf{Z}_r,\mathbf{z}_*}\in\mathbb R^{NN_rN_t}$ and $\mathbf{K}_{\mathbf{Z}_r}=\mathbf{K}_{\boldsymbol{\mu}}\otimes \mathbf{K}_{\mathbf{X}_r}\otimes \mathbf{K}_{\mathbf{T}}$. Direct computation of the quadratic form in the posterior variance is not feasible for high-resolution spatial grids since $\mathbf{K}_{\mathbf{X}_r}$ lacks Kronecker structure. We prove a theoretical result that establishes bounds on the posterior variance which can be efficiently computed using Kronecker matrix algebra.

\begin{lemma}
    \label{lem:gappy_posterior_variance_bounds}
    The posterior variance at a test point $\mathbf{z}_*$, $\mathrm{Var}(\mathbf{z}_*)$, considering only the regular observations $\mathbf{y}_r$, is bounded by: 
    \begin{equation}
    \begin{split}
        \mathrm{Var}_{\mathbf{Z}}(\mathbf{z}_*)
        \; \le \;
        \mathrm{Var}(\mathbf{z}_*)
        \; \le \;
        k(\mathbf{z}_*,\mathbf{z}_*)
        - 
        \frac{\|\mathbf k_{\boldsymbol{\mu}_*}\|^2\;\|\mathbf k_{t_*}\|^2\;
        \| \widetilde{\mathbf{W}} \mathbf{k}_{\mathbf x_*} \|^2\ }{\lambda_{\max}(\mathbf{K}_{\mathbf{Z}})+\sigma_n^2},
    \end{split}
    \end{equation}
    where $\mathrm{Var}_{\mathbf{Z}}(\mathbf{z}_*) = k(\mathbf{z}_*,\mathbf{z}_*)
    -
    {\mathbf k}(\mathbf{z}_*)^T ( \mathbf{K}_{\mathbf{Z}}+\sigma_n^2 \mathbf{I} )^{-1}{\mathbf k}(\mathbf{z}_*)$ is the posterior variance at a test point $\mathbf{z}_*$ considering the full grid $\mathbf{Z}$, ${\mathbf k}(\mathbf{z}_*) := \mathbf{K}_{\mathbf{Z},\mathbf{z}_*} = \mathbf{k}_{\boldsymbol{\mu}_*}  \otimes \mathbf{k}_{\mathbf{x}_*} \otimes \mathbf{k}_{t_*} \in \mathbb R^{NMN_t}$, $\mathbf{k}_{\boldsymbol{\mu}_*} \in \mathbb{R}^{N}$,
 $\mathbf{k}_{t_*} \in \mathbb{R}^{N_t}$, 
$\mathbf{k}_{\mathbf{x}_*} \in \mathbb{R}^{M}$ are the  cross-covariance vectors induced by the product kernel, $\mathbf{K}_{\mathbf{Z}}
    =
    \mathbf{K}_{\boldsymbol{\mu}}
    \otimes
    (
    \bigotimes_{\ell=1}^{d}\mathbf{K}_{x_\ell}
    )
    \otimes
    \mathbf{K}_{\mathbf{T}} $ is the full covariance matrix over $\mathbf{Z}$, and 
      $\lambda_{\max}(\mathbf{K}_{\mathbf{Z}})$ is the largest eigenvalue of $\mathbf{K}_{\mathbf{Z}}$.
\end{lemma}

\begin{proof}    
    Since $\mathbf{K}_{\mathbf{Z}_r}+\sigma_n^2 \mathbf{I}$ is SPD, a Rayleigh-quotient bound yields
    \begin{equation}
    \mathrm{Var}(\mathbf{z}_*)
    \;\le\;
    k(\mathbf{z}_*,\mathbf{z}_*)
    -
    \frac{\|{\mathbf k}_r(\mathbf{z}_*)\|^2}{\lambda_{\max}(\mathbf{K}_{\mathbf{Z}_r})+\sigma_n^2}.
    \label{eq:rq_bound_r}
    \end{equation}
    Moreover, since $\mathbf{K}_{\mathbf{Z}_r}$ is a principal submatrix of the full-grid matrix
    $\mathbf{K}_{\mathbf{Z}}=\mathbf{K}_{\boldsymbol{\mu}}\otimes\left(\bigotimes_{\ell=1}^d \mathbf{K}_{x_\ell}\right)\otimes \mathbf{K}_{\mathbf{T}}$, it follows that 
    $\lambda_{\max}(\mathbf{K}_{\mathbf{Z}_r})\le \lambda_{\max}(\mathbf{K}_{\mathbf{Z}})$. Substituting this
    into \eqref{eq:rq_bound_r}, we have 
    \begin{equation}
    \mathrm{Var}(\mathbf{z}_*)
    \;\le\;
    k(\mathbf{z}_*,\mathbf{z}_*)
    -
    \frac{\| {\mathbf k}_r(\mathbf{z}_*)\|^2}{\lambda_{\max}(\mathbf{K}_{\mathbf{Z}})+\sigma_n^2}, 
    \label{eq:rq_bound_structured}
    \end{equation}
    where $\lambda_{\max}(\mathbf{K}_{\mathbf{Z}})$ can be computed efficiently using the Kronecker structure as 
    $\lambda_{\max}(\mathbf{K}_{\boldsymbol{\mu}})
    (\prod_{\ell=1}^d \lambda_{\max}(\mathbf{K}_{x_\ell}) )
    \lambda_{\max}(\mathbf{K}_{\mathbf{T}})$. Recall that $\lambda_{\max}(\mathbf{K}_{\boldsymbol{\mu}})$, $\{\lambda_{\max}(\mathbf{K}_{x_\ell})\}_{\ell=1}^d$, and
    $\lambda_{\max}(\mathbf{K}_{\mathbf{T}})$ are available from the eigendecompositions  computed in Section~\ref{subsec:rect_nlml}.
    
Since ${\mathbf k}_r(\mathbf{z}_*)=\mathbf W\,\mathbf k(\mathbf{z}_*)$, where $\mathbf{W}=\mathbf{I}_N\otimes \widetilde{\mathbf{W}}\otimes \mathbf{I}_{N_t}$ is the selection matrix for the regular spatial points, we can express $\|{\mathbf k}_r(\mathbf{z}_*)\|^2$ in terms of the full cross-covariance vector $\mathbf{k}(\mathbf{z}_*)$ as follows:
    \begin{equation}
    \|{\mathbf k}_r(\mathbf{z}_*)\|^2
    =
    \mathbf k(\mathbf{z}_*)^T \mathbf{W}^T\mathbf{W}\,\mathbf{k}(\mathbf{z}_*) =
    \|\mathbf k_{\boldsymbol{\mu}_*}\|^2\;\|\mathbf k_{t_*}\|^2\;
        \| \widetilde{\mathbf{W}} \mathbf{k}_{\mathbf x_*} \|^2\ ,
    \label{eq:norm_selection}
    \end{equation}
    where we used the fact ${\mathbf k}(\mathbf{z}_*) := \mathbf{K}_{\mathbf{Z},\mathbf{z}_*} = \mathbf{k}_{\boldsymbol{\mu}_*}  \otimes \mathbf{k}_{\mathbf{x}_*} \otimes \mathbf{k}_{t_*} $ and the definition of $\mathbf{W}$ to obtain the last step. Substituting \eqref{eq:norm_selection} into \eqref{eq:rq_bound_structured} yields the upper bound on the posterior variance stated in the lemma.

    To prove the lower bound on the posterior variance, we start with the expression for the posterior variance at a test point $\mathbf{z}_*$ based on the full grid $\mathbf{Z}$ containing both regular and gappy points, which is given by:
    \begin{equation}
    \mathrm{Var}_{\mathbf{Z}}(\mathbf{z}_*)
    =
    k(\mathbf{z}_*,\mathbf{z}_*)
    -
    {\mathbf k}(\mathbf{z}_*)^T \left(\mathbf{K}_{\mathbf{Z}}+\sigma_n^2 \mathbf{I}\right)^{-1}{\mathbf k}(\mathbf{z}_*).
    \label{eq:var_true_def_full}
    \end{equation}
    
   As in \eqref{eq:block_row}, we partition $\left( \mathbf{K}_{\mathbf{Z}}+\sigma_n^2 \mathbf{I} \right)$ into blocks corresponding to regular (r) and gappy (g) index sets using orthogonal permutation matrix $\mathbf{P} = [\mathbf{W}^T, \mathbf{V}^T]^T$, induced by the selection matrices $(\mathbf{W} , \mathbf{V})$:
    \begin{equation}
        \mathbf{A} = \mathbf{P} \left( \mathbf{K}_{\mathbf{Z}}+\sigma_n^2 \mathbf{I} \right) \mathbf{P}^T 
        = \begin{bmatrix}
    \mathbf{A}_{rr} & \mathbf{A}_{rg} \\
    \mathbf{A}_{gr} & \mathbf{A}_{gg}
    \end{bmatrix},
    \end{equation}
    where $\mathbf{A}_{rr} = \mathbf{W}  (\mathbf{K}_{\mathbf{Z}}+\sigma_n^2 \mathbf{I})\mathbf{W}^T = \mathbf{K}_{\mathbf{Z}_r}+\sigma_n^2 \mathbf{I}$ is the covariance matrix for the regular points $\mathbf{Z}_r$. Similarly, we rewrite the cross-covariance vector $\mathbf{k}(\mathbf{z}_*)$ as a block vector containing the cross-covariance vectors for the regular and gappy points: 
    \begin{equation}
        \begin{bmatrix}
        \mathbf{k}_r \\
        \mathbf{k}_g
        \end{bmatrix} = \mathbf{P} \mathbf{k}(\mathbf{z}_*) = \begin{bmatrix}
            \mathbf{W} \mathbf{k}(\mathbf{z}_*) \\
            \mathbf{V} \mathbf{k}(\mathbf{z}_*) 
        \end{bmatrix},
    \end{equation}

    Using the block matrix inversion lemma~\cite{horn2012matrix}, the inverse of the $\mathbf{A}$ can be written as follows:
    \begin{equation}
        \mathbf{A}^{-1} 
         = \begin{bmatrix}
            \mathbf{A}_{rr}^{-1} + \mathbf{A}_{rr}^{-1} \mathbf{A}_{rg} \mathbf{S}^{-1} \mathbf{A}_{gr} \mathbf{A}_{rr}^{-1} & 
            -\mathbf{A}_{rr}^{-1} \mathbf{A}_{rg} \mathbf{S}^{-1} \\
            -\mathbf{S}^{-1} \mathbf{A}_{gr}  \mathbf{A}_{rr}^{-1} &
            \mathbf{S}^{-1}
        \end{bmatrix},
        \label{eq:block_matrix_inverse}
    \end{equation}
    where $\mathbf{S} = \mathbf{A}_{gg} - \mathbf{A}_{gr} \mathbf{A}_{rr}^{-1} \mathbf{A}_{rg}$ is the Schur complement of $\mathbf{A}_{rr}$ in $\mathbf{A}$.

    Substituting $\mathbf{k}(\mathbf{z}_*) =  \mathbf{P}^T \begin{bmatrix}
        \mathbf{k}_r \\
        \mathbf{k}_g
    \end{bmatrix}  $ in \eqref{eq:var_true_def_full}, we get:
    \begin{equation}
    \begin{split}
        \mathrm{Var}_{\mathbf{Z}}(\mathbf{z}_*)
        =
        k(\mathbf{z}_*,\mathbf{z}_*)
        -
        \begin{bmatrix}
            \mathbf{k}_r^T &
            \mathbf{k}_g^T
        \end{bmatrix} \mathbf{P} 
        \left(\mathbf{K}_{\mathbf{Z}}+\sigma_n^2 \mathbf{I}\right)^{-1} \mathbf{P}^T
        \begin{bmatrix}
            \mathbf{k}_r \\
            \mathbf{k}_g
        \end{bmatrix} 
    \end{split}
    \label{eq:var_true_def_full_substituted}
    \end{equation}

    We note that $\mathbf{A}^{-1} = \mathbf{P} \left( \mathbf{K}_{\mathbf{Z}}+\sigma_n^2 \mathbf{I} \right)^{-1} \mathbf{P}^T$ and substituting \eqref{eq:block_matrix_inverse} into \eqref{eq:var_true_def_full_substituted}, we get:
    \begin{equation}
        \begin{split}
        \mathrm{Var}_{\mathbf{Z}}(\mathbf{z}_*)
        &=
        k(\mathbf{z}_*,\mathbf{z}_*)
        -
        \mathbf{k}_r^T \left(\mathbf{A}_{rr}^{-1} + \mathbf{A}_{rr}^{-1} \mathbf{A}_{rg} \mathbf{S}^{-1} \mathbf{A}_{gr} \mathbf{A}_{rr}^{-1} \right) \mathbf{k}_r  \\
        & \quad \quad + 2 \mathbf{k}_r^T \left( \mathbf{A}_{rr}^{-1} \mathbf{A}_{rg} \mathbf{S}^{-1} \right) \mathbf{k}_g - \mathbf{k}_g^T \mathbf{S}^{-1} \mathbf{k}_g \\
        &= \underbrace{ k(\mathbf{z}_*,\mathbf{z}_*)
        - \mathbf{k}_r^T \mathbf{A}_{rr}^{-1} \mathbf{k}_r }_{ \mathrm{Var}(\mathbf{z}_*) } -   \underbrace{\left( \mathbf{u} - \mathbf{k}_g \right)^T \mathbf{S}^{-1} \left( \mathbf{u} - \mathbf{k}_g \right)}_{\Delta \mathbf{Q}}  \\
        \end{split}
    \end{equation}
    where $ \mathbf{u} = \mathbf{A}_{gr} \mathbf{A}_{rr}^{-1} \mathbf{k}_r $ and  $\Delta \mathbf{Q}$ is the quadratic term. Since $\mathbf{A}$ is SPD, its Schur complement $\mathbf{S}$ and $\mathbf{S}^{-1}$ are also SPD. Therefore, we have $\Delta \mathbf{Q} \ge 0$ and hence
$\mathrm{Var}_{\mathbf{Z}}(\mathbf{z}_*) \le \mathrm{Var}(\mathbf{z}_*)$. This completes the proof of the lower bound on the posterior variance.
\end{proof}

Lemma~\ref{lem:gappy_posterior_variance_bounds} provides the bounds on posterior variance $\mathrm{Var}(\mathbf{z}_*)$ for a single test point $\mathbf{z_*}$. In practice, however, we often require the posterior variance  $\mathrm{Var}(\mathbf{Z}_*) \in \mathbb{R}^{M_* N_{t*}}$ over the full spatio-temporal test grid $\mathbf{Z}_{*}$. We show how to efficiently compute the bounds on the (diagonal) posterior variance over the full test grid using Kronecker matrix algebra.
First, the lower bound on $\mathrm{Var}(\mathbf{Z}_*)$ coincides with the posterior variance $\mathrm{Var}_{\mathbf{Z}}(\mathbf{Z}_*)$ computed using the full training grid $\mathbf{Z}$. This can be efficiently computed using the approach in Section~\ref{subsec:rect_posterior} with a computational complexity of $\mathcal{O}( (N + dM^{\frac{1}{d}} + N_t)NMN_t )$ operations, assuming the test grid coincides with the training grid. We can then extract the posterior variance $\mathrm{Var}_{\mathbf{Z}}(\mathbf{Z}_{r*})$ at regular test points $\mathbf{Z}_{r*}$ by applying the selection matrix $\widetilde{\mathbf{W}}_*$ to $\mathrm{Var}_{\mathbf{Z}}(\mathbf{Z}_*)$, i.e.,
$ \mathrm{Var}_{\mathbf{Z}}(\mathbf{Z}_{r*})
    = (\widetilde{\mathbf{W}}_* \otimes \mathbf{I}_{N_{t*}}) \mathrm{Var}_{\mathbf{Z}}(\mathbf{Z}_*),$
where $\mathrm{Var}(\mathbf{Z}_{r*}) \in \mathbb{R}^{N_{r*} N_{t*}}$. 

Next, using the inequality \eqref{eq:rq_bound_structured} from Lemma~\ref{lem:gappy_posterior_variance_bounds}, the upper bound on the posterior variance over $\mathbf{Z}_{*}$ can be written in vectorized form as:
\begin{equation}
    \mathrm{Var}(\mathbf{Z}_*)
    \;\le\;
    \mathrm{diag} ( \mathbf{K}_{\mathbf{Z}_*,\mathbf{Z}_*} )
    -
    \frac{  \left[ q^{(1)}, q^{(2)},   \ldots,   q^{(M_* N_{t*})} \right]^T}{\lambda_{\max}(\mathbf{K}_{\mathbf{Z}})+\sigma_n^2},
    \label{eq:vector_bound_var}
\end{equation}
where $q^{(i)} = \| {\mathbf k}_r(\mathbf{z}_*^{(i)})\|^2 $ and $ {\mathbf k}_r(\mathbf{z}_*^{(i)}) := \mathbf{K}_{\mathbf{Z}_r,\mathbf{z}_*^{(i)}} \in \mathbb{R}^{N N_r N_t}$ is the cross-covariance vector between the test point $\mathbf{z}_*^{(i)}$ and the regular points $\mathbf{Z}_r$.

The first term, $\mathrm{diag}(\mathbf{K}_{\mathbf{Z}_{*}, \mathbf{Z}_*})$, in the above equation can be efficiently computed using the identity $\mathrm{diag}(\mathbf{A} \otimes \mathbf{B}) = \mathrm{diag}(\mathbf{A}) \otimes \mathrm{diag}(\mathbf{B})$ (for square matrices $\mathbf{A}$ and $\mathbf{B}$) as follows
\begin{equation}
\mathrm{diag}(\mathbf{K}_{\mathbf{Z}_*,\mathbf{Z}_*})
=
k_{\boldsymbol{\mu}}(\boldsymbol{\mu}_*,\boldsymbol{\mu}_*)
\left ( \left ( \sbigotimes_{\ell=1}^d \mathrm{diag} \left ( \mathbf{K}_{x_{*,\ell},x_{*,\ell}} \right ) \right )
\otimes \; 
\mathrm{diag}\left ( \mathbf{K}_{\mathbf{T}_*,\mathbf{T}_*} \right ) 
\right ).
\label{var_first_term_computation}
\end{equation}

The second term in \eqref{eq:vector_bound_var} can be written as $ \left[ q^{(1)}, q^{(2)},   \ldots,   q^{(M_* N_{t*})} \right]^T = $  $ \mathrm{diag} (\mathbf{K}_{\mathbf{Z}_*,\mathbf{Z}_r}  \mathbf{K}_{\mathbf{Z}_*,\mathbf{Z}_r}^T )$. Using the relation $\mathbf{K}_{\mathbf{Z}_*,\mathbf{Z}_r}= \mathbf{K}_{\mathbf{Z}_*,\mathbf{Z}} \, \mathbf{W}^T$, we can rewrite this term as
$
    \mathrm{diag} (\mathbf{K}_{\mathbf{Z}_*,\mathbf{Z}_r} \, \mathbf{K}_{\mathbf{Z}_*,\mathbf{Z}_r}^T ) =
    \mathrm{diag} ( \mathbf{K}_{\mathbf{Z}_*,\mathbf{Z}} \, \mathbf{W}^T \, \mathbf{W} \, \mathbf{K}_{\mathbf{Z}_*,\mathbf{Z}}^T ).
$

Since $\mathbf{W}$ is a selection matrix, we have $\mathbf{W}^T \mathbf{W} = \mathbf{P}_r$, where $\mathbf{P}_r$ is a diagonal matrix with ones on the regular indices and zeros on the gaps. Using the identities $\mathrm{diag}(\mathbf{A}\,\mathrm{diag}(\mathbf{d})\,\mathbf{A}^{T})=(\mathbf{A}\odot\mathbf{A})\mathbf{d}$ and $(\mathbf{A}\otimes\mathbf{B})\odot(\mathbf{C}\otimes\mathbf{D})=(\mathbf{A}\odot\mathbf{C})\otimes(\mathbf{B}\odot\mathbf{D})$ used in Section~\ref{subsec:rect_posterior}, we obtain
\begin{eqnarray}
    & & \mathrm{diag} (\mathbf{K}_{\mathbf{Z}_*,\mathbf{Z}_r} \, \mathbf{K}_{\mathbf{Z}_*,\mathbf{Z}_r}^T )  = 
    \left(
     \mathbf{K}_{\mathbf{Z}_*,\mathbf{Z}} 
    \odot
    \mathbf{K}_{\mathbf{Z}_*,\mathbf{Z}} 
    \right) \, \mathrm{diag}(\mathbf{P}_r) \label{eq:vector_upper_bound_second_term_computation} \\
    & & = \left( \left (\mathbf{k}_{\boldsymbol{\mu}_*}^{T} \odot \mathbf{k}_{\boldsymbol{\mu}_*}^{T} \right ) 
    \otimes
    \Bigl ( \sbigotimes_{\ell=1}^{d} \mathbf{K}_{x_{*,\ell},x_\ell} \odot \mathbf{K}_{x_{*,\ell},x_\ell} \Bigr )
    \otimes
    \left ( \mathbf{K}_{\mathbf{T}_*,\mathbf{T}} \odot \mathbf{K}_{\mathbf{T}_*,\mathbf{T}} \right ) 
    \right) \, \mathrm{diag}(\mathbf{P}_r). \nonumber 
\end{eqnarray}
Using \eqref{var_first_term_computation} and \eqref{eq:vector_upper_bound_second_term_computation}, the upper bound on $\mathrm{Var}(\mathbf{Z}_*)$ in \eqref{eq:vector_bound_var} with computational complexity $\mathcal{O}( (N + dM^{\frac{1}{d}} + N_t)NMN_t )$, matching that of the lower bound. The upper bound for the regular test points $\mathbf{Z}_{r*}$ can be extracted by multiplying the result with the selection matrix $(\widetilde{\mathbf{W}}_* \otimes \mathbf{I}_{N_{t*}})$. In summary, computing both the lower and upper bounds on the posterior variance $\mathrm{Var}(\mathbf{Z}_{r*})$ incurs a computational cost comparable to that of the posterior mean.

\subsection{Extension to parametrized spatial domains}
\label{subsec_parametrized_spatial_domains}

\begin{figure}[t]
    \centering
    \begin{tikzpicture}
    \node (img1) {\includegraphics[width=5cm]{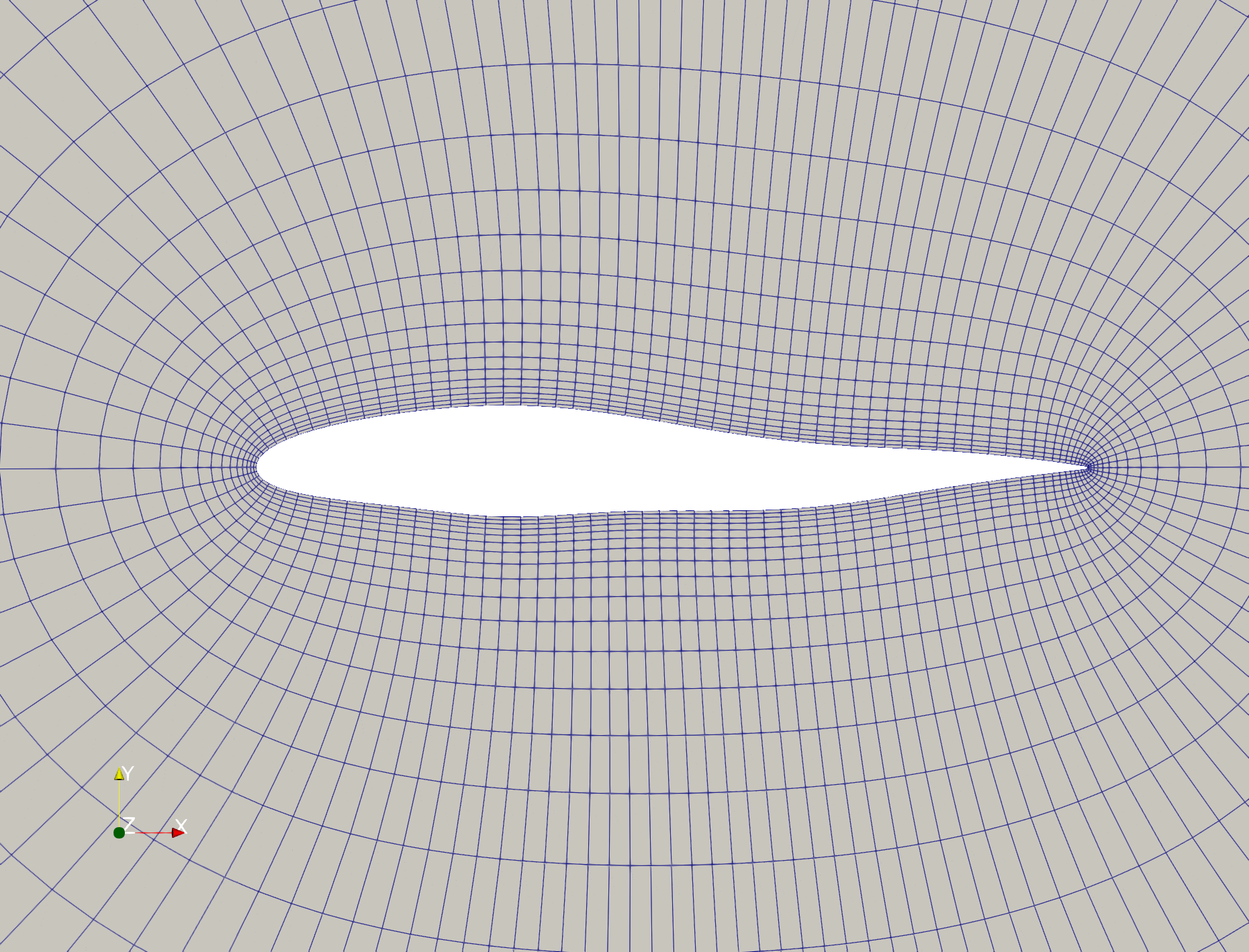}};
    \node[below=0.02cm of img1, inner sep=0pt] (label1) {$\Omega(\boldsymbol{\mu}_{\Omega})$};
    \node[right=2cm of img1] (img2) {\includegraphics[width=5cm]{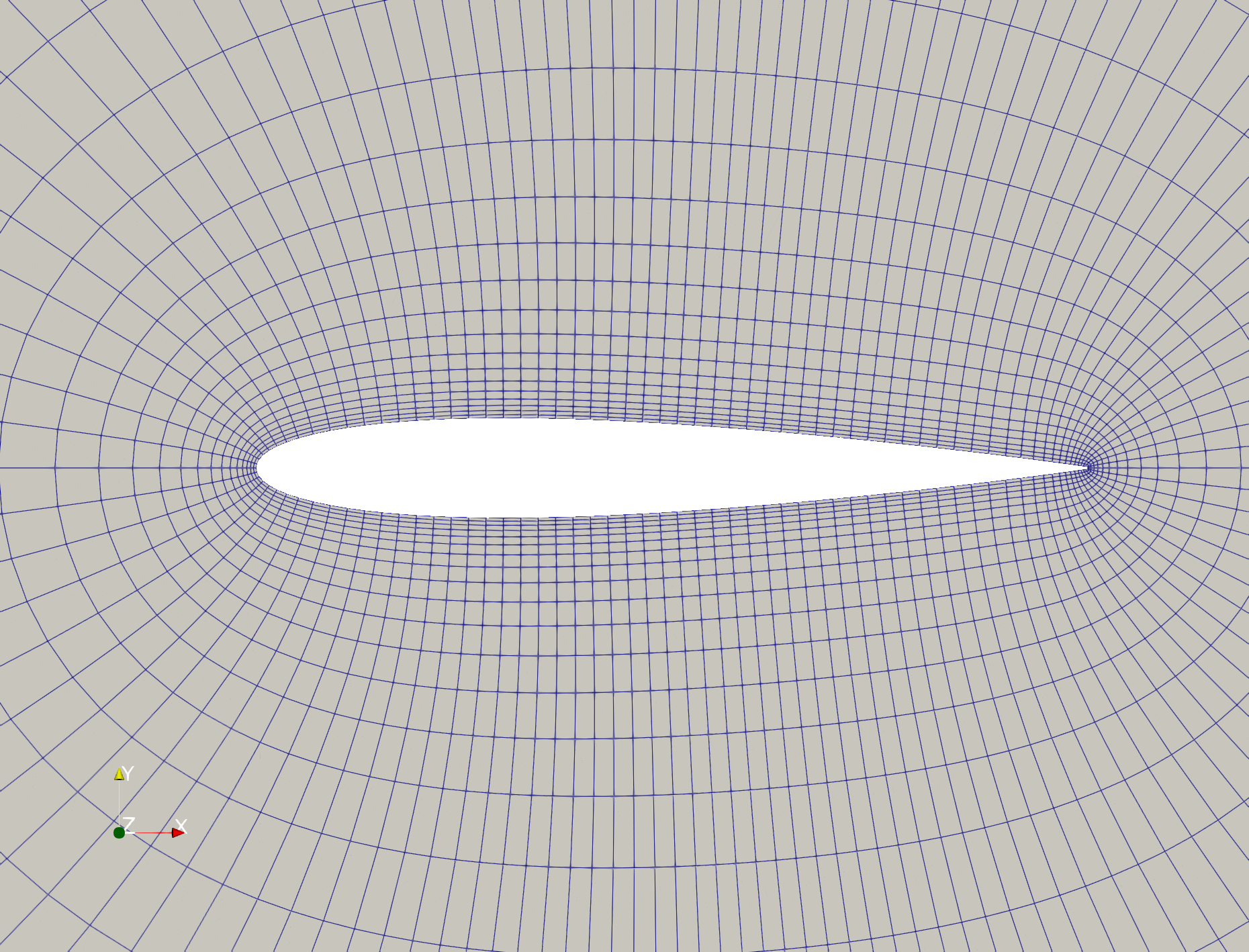}};
    \node[below=0.02cm of img2, inner sep=0pt] (label2) {$\tilde{\Omega}$};
    \draw[->] (img1) -- (img2);
    \end{tikzpicture}
    \caption{Illustration of the geometric transformation: mapping a parametrized physical domain $\Omega(\boldsymbol{\mu}_{\Omega})$ (left) to a fixed reference domain 
    $\tilde{\Omega}$ (right) to enable tensor-product grid embedding.}
    \label{fig:ref_domain_map}
\end{figure}
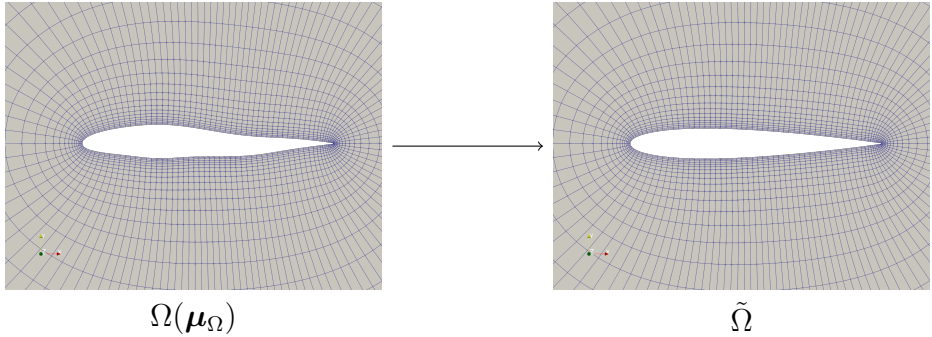

Thus far we have assumed that, after embedding, the learning problem is posed on a single rectilinear reference grid. In many applications (e.g., aerodynamic shape optimization), the computational domain depends on geometric parameters $\boldsymbol{\mu}_{\Omega}$, and snapshots are obtained on parameter-dependent meshes. For example, consider a PDE defined over a domain parametrized by $\boldsymbol{\mu}_{\Omega}\in\mathcal{P}_{\Omega}\subset\mathbb{R}^{D_{\Omega}}$:
\begin{equation}
    \begin{aligned} 
        \frac{\partial{u(\mathbf{x}, t)}}{\partial{t}} + \mathcal{N}(u(\mathbf{x}, t)) = f(\mathbf{x}, t) \quad \forall \, \left ( \mathbf{x}, t \right ) \in \Omega(\boldsymbol{\mu}_{\Omega}) \times \mathcal{T}. \nonumber
    \end{aligned}  
\end{equation}
When dealing with parametrized domains, a standard approach~\cite{mohan2011stochastic, manzoni2016mathicse} is to map the spatial domain for each geometry parameter $\boldsymbol{\mu}_{\Omega}$ to a fixed reference domain $\tilde{\Omega}$. This step leads to a parametrized PDE defined over the fixed reference domain, which takes the form 
\begin{equation}
    \begin{aligned} 
        \frac{\partial{u(\tilde{\mathbf{x}}, t; \boldsymbol{\mu}_{\Omega} )}}{\partial{t}} + \mathcal{N}(u( \tilde{\mathbf{x}}, t; \boldsymbol{\mu}_{\Omega} ); \boldsymbol{\mu}_{\Omega}) = f(\tilde{\mathbf{x}}, t; \boldsymbol{\mu}_{\Omega}), \; \forall (\tilde{\mathbf{x}}, t) \in \tilde{\Omega} \times \mathcal{T}, \nonumber 
    \end{aligned}  
\end{equation}
where $\tilde{\mathbf{x}} = \Phi(\mathbf{x}; \boldsymbol{\mu}_{\Omega}) $ with $\Phi(\cdot ; \boldsymbol{\mu}_{\Omega}): \Omega(\boldsymbol{\mu}_{\Omega}) \rightarrow \tilde{\Omega} $ denoting a mapping from the parametrized domain to the fixed reference domain.

The transformed PDE can be solved in the reference domain where $\boldsymbol{\mu}_{\Omega}$ can be treated as a domain-independent parameter, as in equation \eqref{unsteady-pde}. In other words, the learning problem can be  framed within the reference domain, with the predictions mapped back to the parametrized domain. 
When the mapped reference grid is rectilinear and complete, the rectilinear-grid method applies directly; when the mapping produces gaps (as is typical for complex geometries), the gappy-grid extension can be used to retain scalability. We illustrate this approach later in our numerical studies.

%% file: chapters/Results.tex
\section{Results}\label{sec_results}

This section presents a comprehensive set of numerical experiments on several challenging benchmark problems to validate the accuracy and scalability of the proposed framework. These test cases span a range of physical phenomena, including fluid dynamics and solid mechanics, and include data on both rectilinear and unstructured grids, as well as problems with parametrized spatial domains. We evaluate the performance of our GP framework equipped with the deep product kernel, and compare it against established physics-based reduced-order models and state-of-the-art data-driven methods such as Fourier Neural Operators \cite{kovachki2023neural, li2020fourier} and DeepONets \cite{lu2021learning, lu2022comprehensive, goswami2022deep}. For problems in which the input parameter vector is high-dimensional, we first employ Principal Component Analysis (PCA) to project the parameters onto a lower-dimensional latent space before training the GP models. Additional details about the neural network architecture used in the deep product kernel and the training setup are provided in \ref{Appendix:Additional_details}.

\subsection{1D unsteady Burgers’ problem}\label{subsec_burgers}

We evaluate the proposed GP framework on a one-dimensional parametrized inviscid Burgers' problem \cite{rewienski2003trajectory, Lee_Carlberg_2019}. The governing PDE for this problem, along with its initial and boundary conditions, is given by
\begin{equation}
    \begin{aligned}
        \frac{\partial{u(x, t)}}{\partial{t}} + \frac{\partial{f(u(x, t)})}{ \partial x} &= 0.02e^{\mu_2 x}, \forall x \in [0, 100], \forall t \in [0, T], \\
        u(0, t) = \mu_1, \forall t \in [0, T] \quad \textrm{and} & \quad u(x, 0) = 1, \forall x \in [0, 100],
    \end{aligned}
\end{equation}
where $f(u) = 0.5u^2$, $u$ is the conserved quantity, and $\boldsymbol{\mu} = (\mu_1, \mu_2)$ is the two-dimensional parameter vector. To facilitate a direct comparison with the physics-based methods reported in \cite{Lee_Carlberg_2019}, we adopt the following parameter configuration:  $(\mu_1, \mu_2) \in [4.25, 5.50] \times [0.015, 0.03]$, final time $T=35$, and time step $\Delta t=0.07$. Each solution snapshot of $u$ for a given parameter vector $\boldsymbol{\mu}$ is defined on a 1D uniformly spaced grid of size $M=256$ over $N_t = 500$ time steps. We use $N=80$ snapshots for training, corresponding to the parameters
$(\mu_1, \mu_2) = \{(4.25 + (1.25/9)i, 0.015 + (0.015/7)j)\}_{i=0, j=0}^{9, 7}$. This results in a total of 10.24 million training points ($N M N_t$). To evaluate the performance, we use the following relative error metric,
\begin{equation}
    \begin{aligned} 
        l_2 \; \textrm{relative test error} = \frac{\| \mathbf{u}_h(\boldsymbol{\mu}^{(i)}) - \widehat{\mathbf{u}}(\boldsymbol{\mu}^{(i)})  \|_{2}}{\| \mathbf{u}_h(\boldsymbol{\mu}^{(i)})\|_{2}},
        \label{relative_error}
    \end{aligned} 
\end{equation}
where $\mathbf{u}_h(\boldsymbol{\mu}^{(i)})$ and $\widehat{\mathbf{u}}(\boldsymbol{\mu}^{(i)})$ represent high-fidelity and predicted snapshots for a given test parameter $\boldsymbol{\mu}^{(i)}$, respectively.

\begin{figure}[ht]
    \centering
    \begin{subfigure}[b]{0.49\linewidth}
        \includegraphics[width=\linewidth]{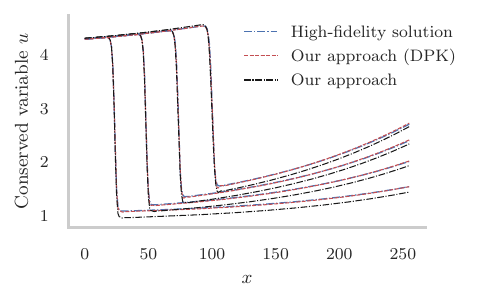}
        \caption{test parameter $\boldsymbol{\mu}^{(1)}_{test}=[4.3, 0.021]$}
        \label{fig:sub1}
    \end{subfigure}
    \hfill
    \begin{subfigure}[b]{0.49\linewidth}
        \includegraphics[width=\linewidth]{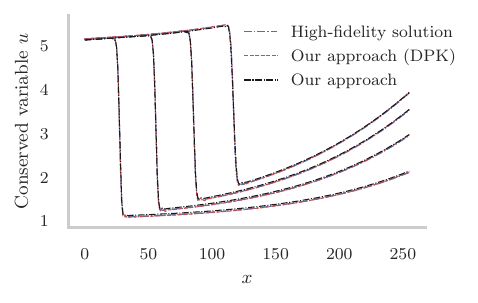}
        \caption{test parameter $\boldsymbol{\mu}^{(2)}_{test}=[5.15, 0.0285]$}
        \label{fig:sub2}
    \end{subfigure}
    \caption{\textit{1D Burgers' problem.} Predicted solutions using the proposed GP framework with DPK and Mat\'ern-$5/2$ based product structured kernel at times $t = 3.5, 7.0, 10.5,$ and $14$ are shown.}
    \label{fig:1D_burgers_results}
\end{figure}

\begin{table}[ht]
    \small
    \centering
    \begin{tabular}{c|cc}
        \hline
        \textbf{Method}	& \multicolumn{2}{c}{\textbf{Relative error}} \\ 
        & $\boldsymbol{\mu}^{(1)}_{test} $  & $\boldsymbol{\mu}^{(2)}_{test} $ \\ \hline \hline
        Deep-LSPG ($p=20$) & \textbf{0.0009} & \textbf{0.001} \\ 
        Proposed GP (DPK-Mat\'ern-$5/2$) & \textbf{0.0033} & \textbf{0.0029} \\ 
        Deep-Galerkin ($p=20$) & 0.015 & 0.013 \\ 
        Proposed GP (Mat\'ern-$5/2$) & 0.0211	& 0.0078\\ 
        POD-Galerkin ($p=20$) &	0.032 &	0.03 \\ 
        POD-LSPG ($p=20$) &	0.03 &	0.029 \\ \hline
    \end{tabular}
    \caption{\textit{1D Burgers' problem.} Relative test errors for the test snapshots corresponding to the parameters $\boldsymbol{\mu}^{(1)}_{test} = (4.3, 0.021)$ and $\boldsymbol{\mu}^{(2)}_{test} = (5.15, 0.0285)$ are presented. Errors for the physics-based methods (POD-Galerkin, POD-LSPG, Deep-Galerkin, and Deep-LSPG) for reduced dimension $p=20$ are obtained from Figure~3 in \cite{Lee_Carlberg_2019}.}
    \label{table:1D_burgers_results}
\end{table}

\begin{figure}[ht]
    \centering
    \begin{subfigure}[b]{0.49\linewidth}
        \includegraphics[width=\linewidth]{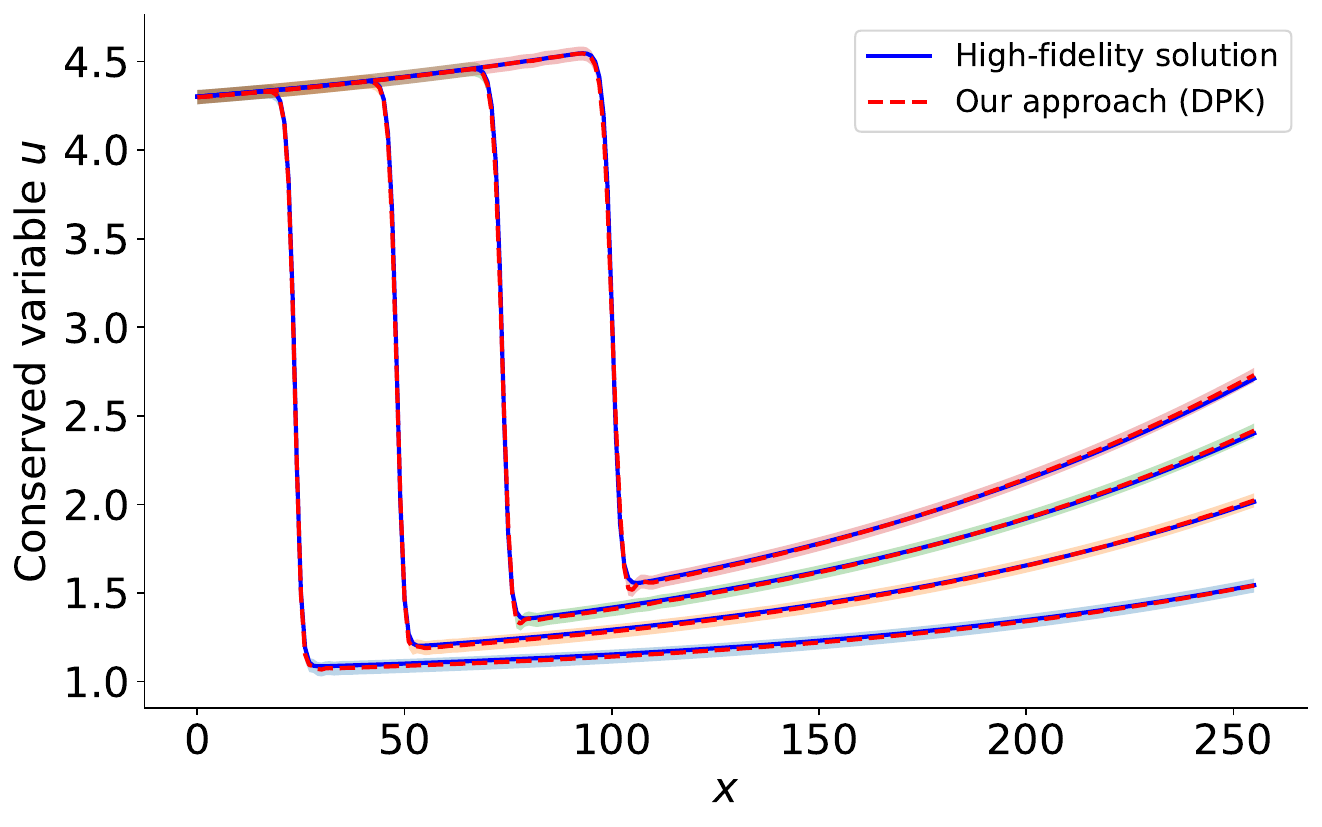}
    \end{subfigure}
    \hfill
    \begin{subfigure}[b]{0.49\linewidth}
        \includegraphics[width=\linewidth]{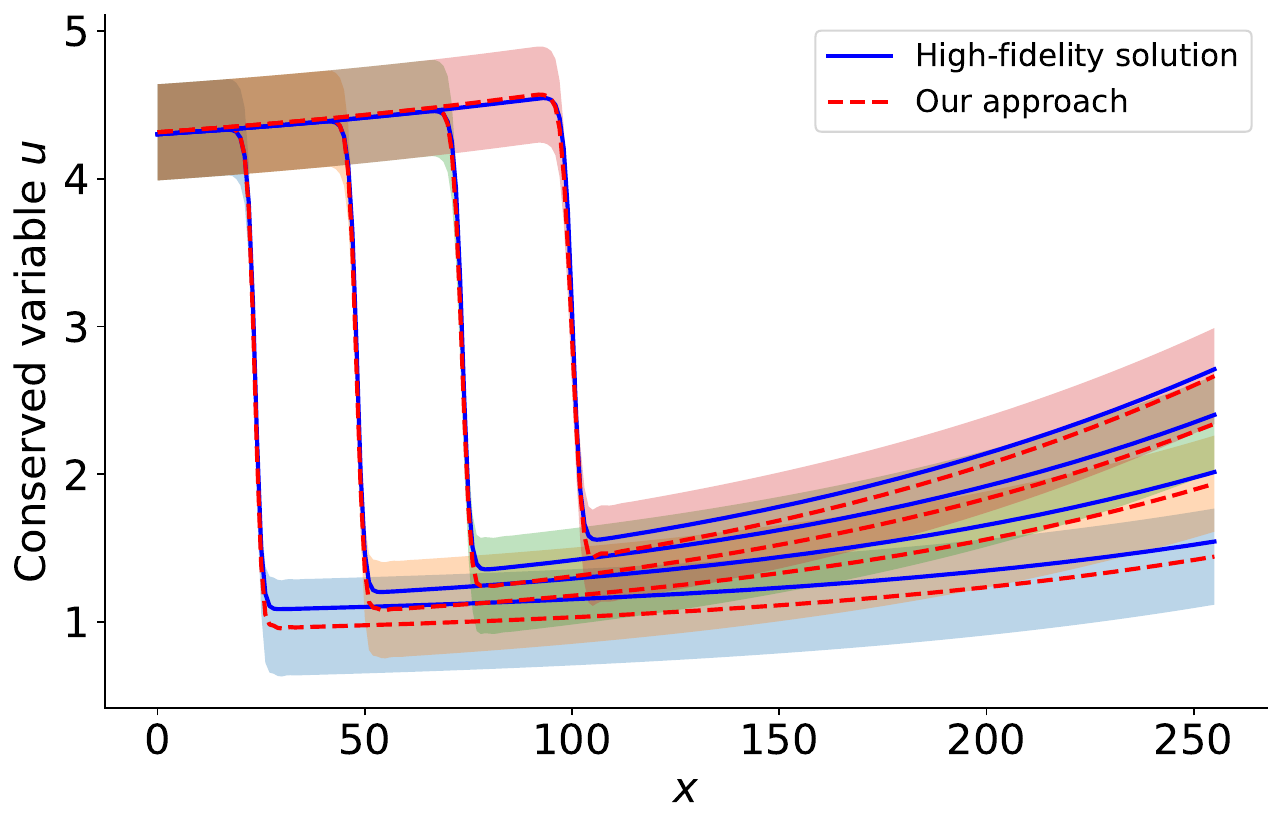}
    \end{subfigure}
    \caption{\textit{1D Burgers' problem.} For test parameter $\boldsymbol{\mu}^{(1)}_{test}=[4.3, 0.021]$, predicted solutions and confidence bounds ($\pm 2\sigma$) using the proposed GP framework with DPK (left) and Mat\'ern-$5/2$ based product kernel (right) at times $t = 3.5, 7.0, 10.5,$ and $14$ are shown.}
    \label{fig:1D_burgers_results_test_0_dkl_matern}
\end{figure}

\begin{figure}[ht]
    \centering
    \begin{subfigure}[b]{0.49\linewidth}
        \includegraphics[width=\linewidth]{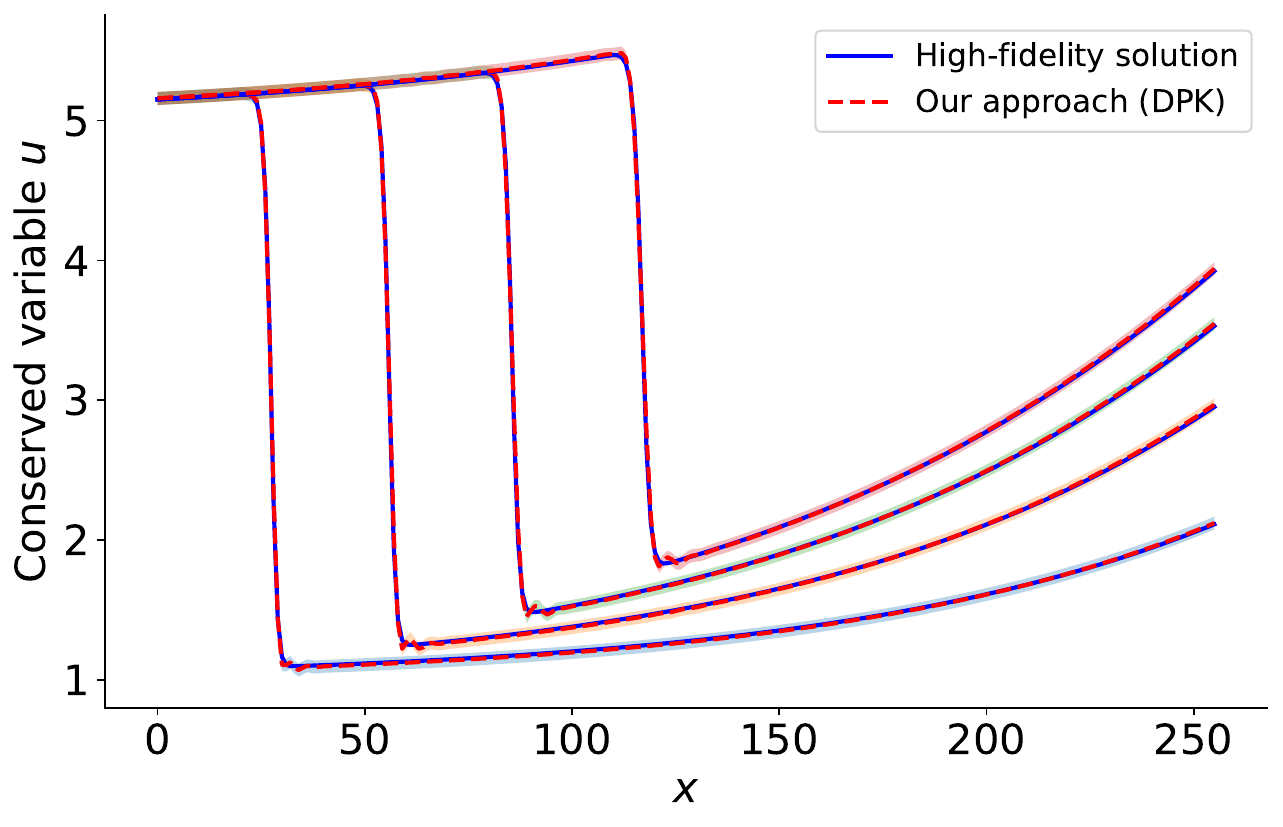}
    \end{subfigure}
    \hfill
    \begin{subfigure}[b]{0.49\linewidth}
        \includegraphics[width=\linewidth]{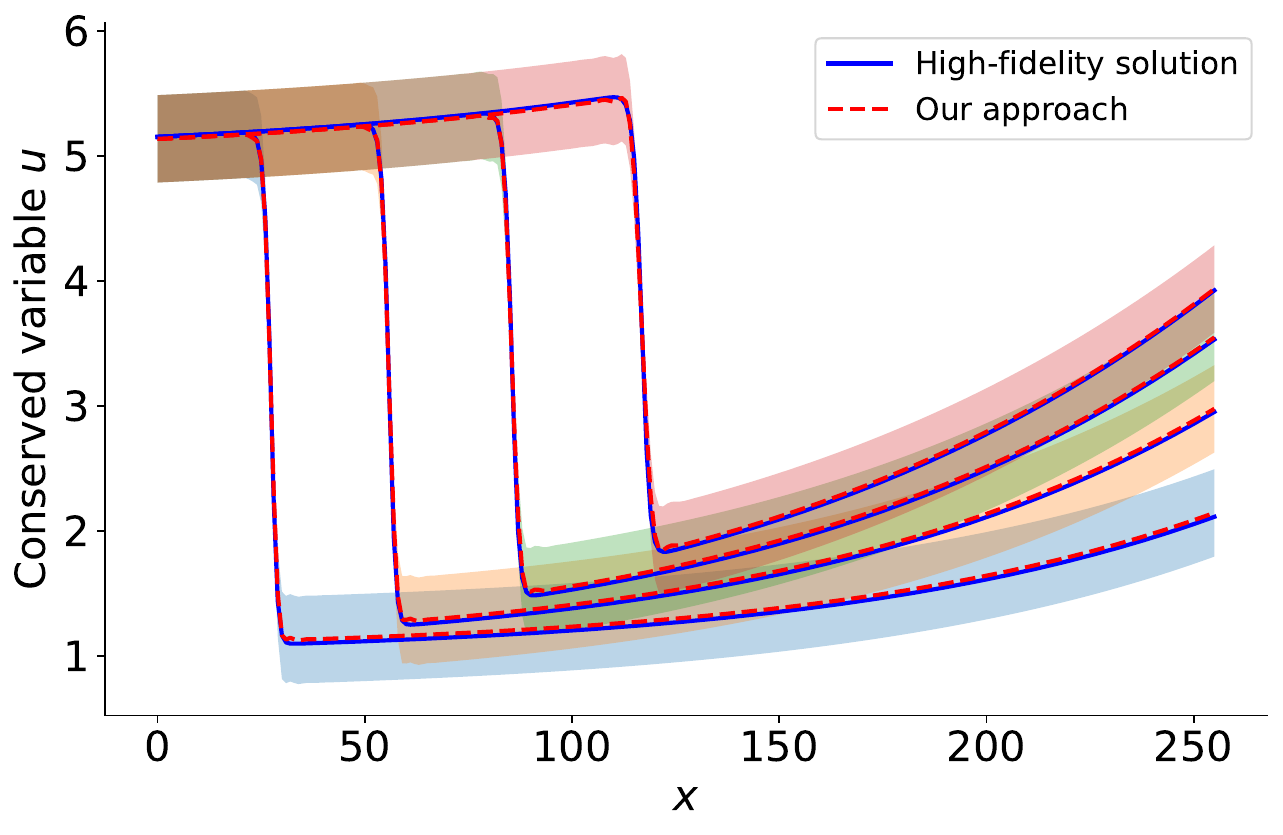}
    \end{subfigure}
     \caption{\textit{1D Burgers' problem.} For test parameter $\boldsymbol{\mu}^{(2)}_{test}=[5.15, 0.0285]$, predicted solutions and confidence bounds ($\pm 2\sigma$) using the proposed GP framework with DPK (left) and Mat\'ern-$5/2$ based product kernel (right) at times $t = 3.5, 7.0, 10.5,$ and $14$ are shown.}
    \label{fig:1D_burgers_results_test_1_dkl_matern}
\end{figure}

The GP framework presented in Section~\ref{section:Struct-GPR} is employed to learn the conserved quantity $u$ as a function of the parameter vector $\boldsymbol{\mu}$, spatial location $x$, and time $t$ using the 80 training snapshots. The Mat\'ern-$\frac{5}{2}$ kernel serves as the base stationary kernel for both the Deep Product Kernel (DPK) and the stationary product-kernel models. For testing, we consider two parameter settings, $\boldsymbol{\mu}^{(1)}_{test}=[4.3, 0.021]$ and $\boldsymbol{\mu}^{(2)}_{test}=[5.15, 0.0285]$. Figure~\ref{fig:1D_burgers_results} shows the predicted solutions of our method with the DPK and Mat\'ern-$5/2$ based product structured kernel for the two test parameters at times $t = 3.5, 7.0, 10.5,$ and $14$. The corresponding relative test errors are reported in Table~\ref{table:1D_burgers_results}. 
As detailed in Table~\ref{table:1D_burgers_results}, our method with DPK significantly outperforms the projection-based methods (POD-Galerkin, POD-LSPG, and Deep-Galerkin) and provides a substantial improvement over the GP with a standard stationary product-kernel. Furthermore, the accuracy of our method is competitive with the Deep-LSPG approach ($p=20$), offering a purely data-driven alternative to physics-based ROMs.

 Figures~\ref{fig:1D_burgers_results_test_0_dkl_matern} and \ref{fig:1D_burgers_results_test_1_dkl_matern} show the predicted solutions and confidence bounds ($\pm 2\sigma$) for the two test parameters at times $t = 3.5, 7.0, 10.5,$ and $14$. The confidence bounds are tighter for the GP with DPK compared to the GP with stationary product kernel. In Appendix~\ref{Appendix:Elastic_block_problem}, we evaluate the learning capacity of the presented GP framework when the training data is increased from 50 to 5000 snapshots. From this study, we observe that our method with the deep product kernel (DPK) shows significantly lower test errors than the data-driven POD-GPR. Furthermore, as the number of training snapshots $N$ increases, the test error of the proposed GP framework consistently decreases, eventually surpassing the accuracy of the physics-based POD-Galerkin, as shown in Figure~\ref{fig:elastic_block_convergence_sub2}.

\subsection{Hyper-elastic problem}\label{subsec_hyperelasticmaterial}

In this problem, we consider a hyper-elastic material on the unit domain $\Omega = [0,1]\times[0,1]$ that contains a void of arbitrary shape at the center, as illustrated in Figure~\ref{fig:elasticity_pred_0} \cite{li2023fourier}. The bottom edge of the unit cell is clamped, and a tensile traction $\mathbf{t}_{\mathrm{trac}} = [0,100]^T$ is applied on the top edge. The prior for the void radius is
$r = 0.2 + \frac{0.2}{1+\exp(\tilde{r})}$ with 
$\tilde{r} \sim \mathcal{N}\!\left(0,\,4^2(-\nabla + 3^2)^{-1}\right)$, 
which satisfies the constraint $r \in [0.2,0.4]$.
The dynamics of the elastic material are governed by the  PDE:
$    \rho \frac{\partial^2 \mathbf{u}}{\partial t^2} + \nabla \cdot \boldsymbol{\sigma} = 0$, where $\rho$ is the mass density, $\mathbf{u}$ is the displacement vector, and $\boldsymbol{\sigma}$ is the stress tensor. Here, the material behavior is characterized by the incompressible Rivlin-Saunders constitutive model, which uses energy density function parameters $C_1 = 1.863 \times 10^5$ and $C_2 = 9.79 \times 10^3$. 

We use the data provided in \cite{li2023fourier} for both training and testing. The training dataset consists of 1000 snapshots, and the testing dataset consists of 200 snapshots generated using a finite element solver and provided on an O-type mesh with 65 and 41 points in the azimuth and radial directions, respectively. To construct the parameter vector $\boldsymbol{\mu}$, we concatenate the $x$ and $y$ coordinates of the deformed mesh nodes and apply PCA to reduce the dimensionality; the coefficients corresponding to the first 20 PCA modes are used to define $\boldsymbol{\mu}$.

The goal is to learn the stress field as a function of the deformed mesh (which encodes the void shape). Since the domain is parametrized, we first map the snapshots onto a fixed reference domain as described in Section~\ref{subsec_parametrized_spatial_domains}, and employ our GP framework presented in Section~\ref{section:general-grids} to learn the stress field in the reference domain. Training snapshots in the reference domain are mapped onto a $120 \times 120$ rectilinear grid, which contains gaps corresponding to points within the void section. During testing, the predicted stress field is mapped back to the deformed O-type mesh with 65 and 41 grid points. For evaluation, we compare relative test errors against machine learning (ML) based operator methods, including Geo-FNO \cite{li2023fourier} and DeepONet \cite{lu2021learning}. From Table~\ref{table:elasticity_darcy}, we observe that our method with a deep product kernel (DPK) outperforms FNO and DeepONet and is slightly better than Geo-FNO. Figure~\ref{fig:elasticity_pred_0} shows a comparison of the predicted stress field with the high-fidelity solution for a sample test geometry.

\begin{table}[h]
  \centering
  \begin{tabular}{c|cc}
  \hline
  \multirow{2}{*}{\textbf{Method}} & \multicolumn{2}{c}{\textbf{Relative error}}  \\
  & training & testing  \\ \hline \hline
  Proposed GP (DPK) & 0.0181 & \textbf{0.0326}  \\
  Geo-FNO (O-mesh) & 0.0344 & 0.0363  \\
  FNO interpolation & 0.0314 & 0.0508 \\
  DeepONet & 0.0528 & 0.0965  \\ \hline
  \end{tabular}
  \caption{\textit{Elasticity problem.} Relative errors for the proposed GP framework compared with ML-based operator methods, including Geo-FNO \cite{li2023fourier} and DeepONet \cite{lu2021learning}. Best testing results are highlighted in bold.}
  \label{table:elasticity_darcy}
\end{table}

\begin{figure}[h]
  \centering
  \begin{subfigure}[b]{0.66\linewidth}
      \includegraphics[width=\linewidth]{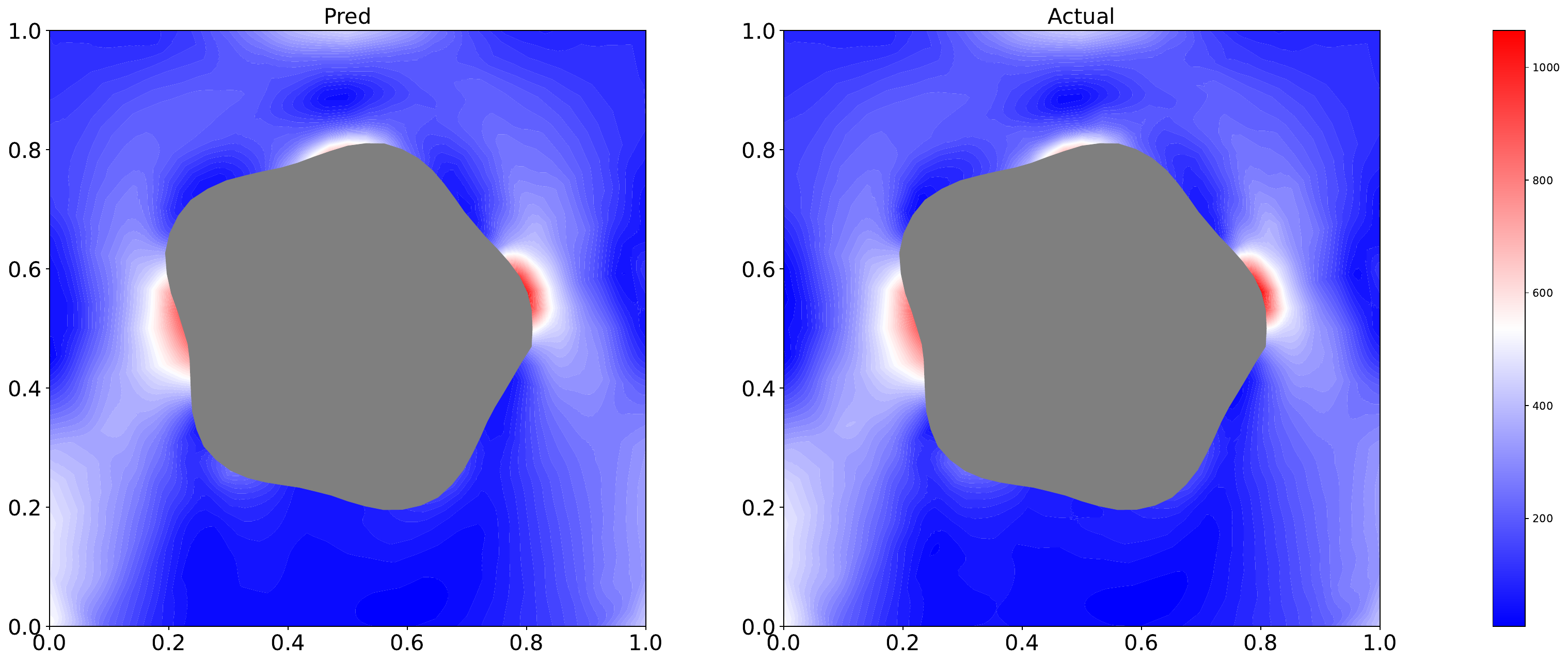}
  \end{subfigure}
  \hfill
  \begin{subfigure}[b]{0.33\linewidth}
      \includegraphics[width=\linewidth]{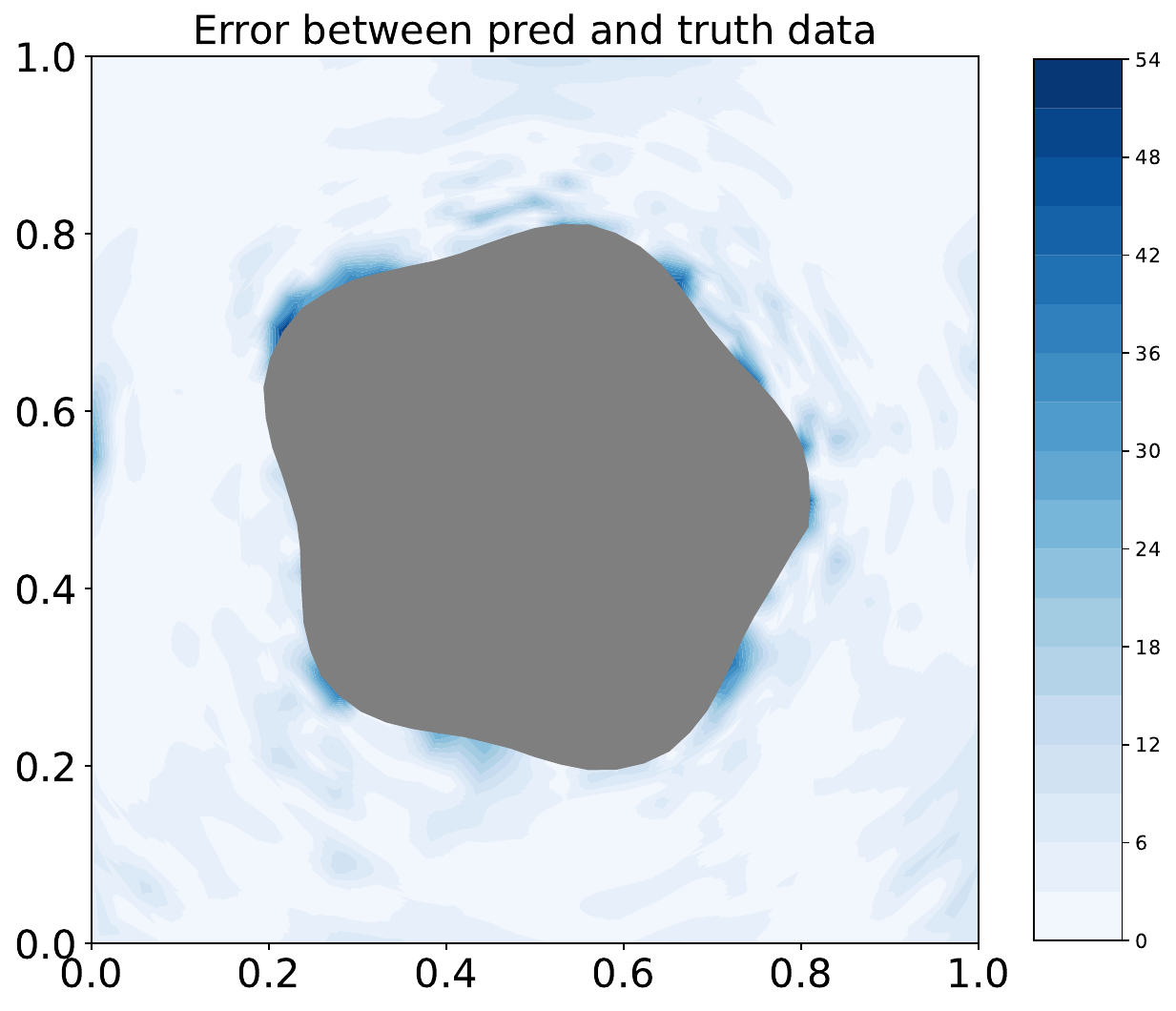}
  \end{subfigure}
  \caption{\textit{Elasticity problem.} Comparison of the stress field predicted by the proposed GP framework (DPK) with the high-fidelity solution for a test geometry in the reference domain.}
  \label{fig:elasticity_pred_0}
\end{figure}

\subsection{2D Transonic flow}\label{subsec_airfoil}

In this problem, we consider a 2D transonic flow around an airfoil with parametrized geometry. The objective is to learn the velocity magnitude field around the airfoil as a function of the airfoil geometry. For shape parameterization, a cubic spline with eight control nodes is used to modify the airfoil shape with respect to the reference airfoil NACA0012. The control nodes are vertically displaced based on samples drawn from a uniform distribution, $d \sim \mathcal{U}(-0.05, 0.05)$. We use the data provided in \cite{li2023fourier} for training and testing. The dataset consists of 1000 training and 200 testing snapshots provided on a C-grid mesh of size $200 \times 50$. To reduce the dimensionality of the parameter vector, we apply PCA to the $y$ coordinates of the deformed airfoil shape and use the coefficients of the first 6 PCA modes as the parameter vector $\boldsymbol{\mu}$.

As in the previous parametrized-domain problem, we map the snapshots onto a fixed reference domain and employ our GP framework presented in Section~\ref{section:general-grids} to learn the velocity magnitude field in the reference domain as a function of the airfoil geometry. First, snapshots in the reference domain are mapped onto a $101 \times 101$ rectilinear grid, which contains gaps corresponding to points within the airfoil section. We use the relative error metric in \eqref{relative_error} for evaluation. The relative error is computed using the predicted field on the $101 \times 101$ rectilinear grid in the reference domain, excluding the gappy region within the airfoil section.

We benchmark our method with other operator methods such as Geo-FNO \cite{li2023fourier} and FNO \cite{li2020fourier}. From the results presented in Table~\ref{table:airfoil_pipe}, it can be observed that the proposed approach with a deep product kernel (DPK) performs better than FNO and similar to Geo-FNO. Furthermore, we notice that the performance of our GP model with stationary Mat\'ern-$5/2$ kernel is better than FNO and UNet.

\begin{table}[h]
    \centering
    \begin{tabular}{c|cc|cc}
    \hline
    \multirow{2}{*}{\textbf{Method}} & \multicolumn{2}{c|}{\textbf{Airfoil}} & \multicolumn{2}{c}{\textbf{Pipe}} \\
    & training & testing & training & testing \\ \hline \hline
    Geo-FNO & 0.0134 & \textbf{0.0138} & 0.0047 & \textbf{0.0067} \\
    Proposed GP (DPK) & 0.0051 & 0.0142 & 0.0062 & 0.0131 \\
    Proposed GP (Mat\'ern-$5/2$) & 0.0065 & 0.0316 & 0.0067 & 0.0171 \\
    FNO interpolation & 0.0287 & 0.0421 & 0.0083 & 0.0151 \\
    UNet interpolation & 0.0039 & 0.0519 & 0.0109 & 0.0182 \\ \hline
    \end{tabular}
    \caption{\textit{Airfoil and pipe problems.} Relative errors comparing the proposed GP framework with ML-based operator methods, including Geo-FNO \cite{li2023fourier} and FNO \cite{li2020fourier}. Best testing results are highlighted in bold.}
    \label{table:airfoil_pipe}
\end{table}

\subsection{Navier-Stokes Equations (Pipe)}\label{subsec_pipe}

In this problem, we consider a 2D incompressible flow in a pipe with parametrized geometry. The objective is to learn the horizontal velocity field in the pipe as a function of the pipe geometry. The length and width of the pipe are 10 and 1, respectively. The pipe shape is described using four piecewise cubic polynomials that define the centerline. These polynomials are characterized at five control points by their slope and vertical position, which are varied according to uniform distributions, $d \sim \mathcal{U}(-2, 2)$ and $d \sim \mathcal{U}(-1, 1)$, respectively. We use the data provided in \cite{li2023fourier} for training and testing. The dataset consists of 1000 training and 200 test snapshots, provided on a $129 \times 129$ mesh. To construct the parameter vector $\boldsymbol{\mu}$, we concatenate the $x$ and $y$ coordinates of the deformed mesh nodes and apply PCA to reduce the dimensionality. The coefficients corresponding to the first 9 PCA modes are used to define $\boldsymbol{\mu}$.

As in the previous problems, the domain is parametrized. We first map the snapshots onto a fixed reference domain and employ our GP framework presented in Section~\ref{section:general-grids} to learn the horizontal velocity in the reference domain as a function of the pipe geometry. Here, snapshots in the reference domain are mapped onto a $129 \times 135$ rectilinear grid. During testing, the predicted velocity is mapped back to the deformed mesh with $129 \times 129$ grid points. For evaluation, we compare relative test errors against Geo-FNO \cite{li2023fourier} and FNO \cite{li2020fourier}. Table~\ref{table:airfoil_pipe} demonstrates that while our method with a deep product kernel (DPK) outperforms FNO and UNet, Geo-FNO attains lower test errors than the proposed GP framework in this case.

%% file: chapters/Conclusions.tex
\section{Concluding remarks}
\label{conclusions}

This paper has presented a scalable GP regression framework for learning solutions of parametrized partial differential equations. At the core of our framework is a deep product kernel that models complex, nonlinear correlations. Each component of the product kernel is a deep kernel, formed by composing a base kernel with a deep neural network to learn spatio-temporal correlations. This kernel structure is central to the method's computational efficiency and predictive performance.

For spatial data on a rectilinear grid, the proposed framework enables the use of Kronecker algebra, which reduces the computational complexity of training and inference from cubic to nearly linear in the number of spatial grid points. In this setting, the posterior variance can be computed exactly and efficiently. We extended the framework to general (unstructured) grids and complex domains by embedding the field onto a rectilinear background grid and treating undefined regions as a gappy data problem. To enable scalable uncertainty quantification under this gappy-grid formulation, we established rigorous theoretical upper and lower bounds for the posterior variance. These bounds are computable at a cost comparable to that of the posterior mean. Moreover, we also showed how our framework can be applied to model spatio-temporal fields over parametrized domains by introducing a mapping to a fixed reference domain.

We benchmarked the proposed framework on a diverse set of problems and observed accuracy that is competitive with, and in several cases superior to, operator learning methods such as FNO and DeepONet. Furthermore, it outperforms traditional physics-based reduced-order models, such as the manifold Galerkin method~\cite{Lee_Carlberg_2019} on the 1D unsteady Burgers' problem. The present framework provides uncertainty estimates alongside its predictions, which are essential for downstream tasks such as design optimization, parameter estimation, and control. 

In summary, the present work demonstrates that scalable Gaussian processes provide a robust and probabilistically principled approach for data-driven surrogate modeling of parametrized PDEs. 
Our work also suggests several promising avenues for future research. A primary direction is to relax the separability implied by the product kernel by developing partially non-separable variants (e.g., coupling selected dimensions or learning low-rank cross-factor corrections). Additional directions include reducing embedding-induced interpolation error on general grids and  improving the efficiency and robustness of the gappy-grid linear solves via preconditioning.

\section*{Acknowledgements}
This work is supported by a Natural Sciences and Engineering Research Council of Canada Discovery Grant.

\bibliographystyle{elsarticle-num} 
\bibliography{bibliography}

\pagebreak

\appendix

%% file: chapters/Appendix.tex
\section{Kronecker algebra} \label{Appendix:Kron_properties}

\textbf{Kronecker product:} If $\mathbf{A}$ is an $m \times n$ matrix and $\mathbf{B}$ is a $p \times q$ matrix, then the Kronecker product $\mathbf{A} \otimes \textbf{B}$ is an $mp \times nq$ matrix
$$
\mathbf{A} \otimes \mathbf{B} = \begin{bmatrix} 
    a_{11}\mathbf{B} & \ldots & a_{1n}\mathbf{B}\\
    \vdots & \ddots & \vdots \\
    a_{m1}\mathbf{B} & \ldots & a_{mn}\mathbf{B}\\
\end{bmatrix}
$$ 

\textbf{Properties of Kronecker Product}
\begin{subequations}
        \begin{align} 
            \textrm{Bilinearity:} \quad & \mathbf{A} \otimes (\textbf{B} + \textbf{C}) = \mathbf{A} \otimes \textbf{B} + \mathbf{A} \otimes \textbf{C} \\
            \textrm{Associativity:} \quad & (\mathbf{A} \otimes \textbf{B}) \otimes \textbf{C} = \mathbf{A} \otimes (\textbf{B} \otimes \textbf{C}) \\
            \textrm{Mixed-product:} \quad & (\mathbf{A} \otimes \textbf{B}) (\mathbf{C} \otimes \textbf{D}) = \mathbf{AC} \otimes \textbf{BD}\\
            \textrm{Inverse:} \quad & (\mathbf{A} \otimes \textbf{B})^{-1}  = \mathbf{A}^{-1} \otimes \textbf{B}^{-1}\\
            \textrm{Transpose:} \quad & (\mathbf{A} \otimes \textbf{B})^{T}  = \mathbf{A}^{T} \otimes \textbf{B}^{T}\\
            \textrm{Trace:} \quad & \textrm{tr}(\mathbf{A} \otimes \textbf{B})  = \textrm{tr}(\mathbf{A}) \textrm{tr}(\textbf{B})\\
            \textrm{Determinant:} \quad & \textrm{det}(\mathbf{A} \otimes \textbf{B})  = (\textrm{det}\mathbf{A})^{N_A}  (\textrm{det}\textbf{B})^{N_B}\\
            \textrm{Vec:} \quad & \textrm{vec}(\mathbf{CXB^T}) = (\mathbf{B} \otimes \mathbf{C}) \textrm{vec}(\textbf{X})\\
            \textrm{Hadamard-product:} \quad & (\mathbf{A} \otimes \mathbf{B}) \odot (\mathbf{C} \otimes \mathbf{D}) = (\mathbf{A} \odot \mathbf{C}) \otimes (\mathbf{B} \odot \mathbf{D}) 
        \end{align} 
    \end{subequations}

\section{Additional experiment details} \label{Appendix:Additional_details}

\subsection{Deep Kernel architecture}

In this section, we describe the architecture of the neural network used in the deep product kernel for the numerical experiments in section \ref{sec_results}. Each component kernel within the deep product kernel utilizes a fully connected neural network. This network consists of three hidden layers with 1000, 500, and 50 neurons, respectively, followed by an output layer with $d_o$ neurons. The output dimension $d_o$ is set to $D$ for the parameter kernel $k_{\mu}$ and $2$ for the spatial $k_{x_i}$ and temporal $k_t$ kernels. This lifting to a two-dimensional latent space provides additional flexibility for learning nonstationary correlation structures along each coordinate. 
A ReLU activation function is applied after each hidden layer. This architecture is the same as that presented in \cite{wilson2016deep}, where the authors employed a non-separable kernel structure with a single neural network.

\subsection{Training details}

For all numerical experiments, we utilized the Adam optimizer \cite{kingma2014adam}, with a weight decay of $2.5 \times 10^{-5}$, $\beta_1 = 0.5$, and $\beta_2 = 0.9$.  For all GP models, the initial noise variance was set to $5 \times 10^{-3}$. In problems that involved the GP methodology presented in Section \ref{section:general-grids} for general grids, the iterative solver used to estimate the pseudo-values $\mathbf{y}_g$ was limited to 2000 iterations and a tolerance of $10^{-5}$. The learning rate and the maximum number of training iterations for all problems are provided below.

\begin{table}[ht]
    \small
    \centering
    \begin{tabular}{|c|c|c|}
        \hline
        \textbf{Problem}
        &  \shortstack[t]{Learning rate} & \shortstack[t]{Maximum \\ training iterations} \\ \hline \hline
        1D unsteady Burgers’ & $10^{-2}$ & 1000 \\ \hline
        Hyper-elastic problem & $1.5 \times 10^{-2}$ & 450 \\ \hline
        2D Transonic flow & ${10^{-2}}^\dagger$ & 368 \\ \hline
        Navier-Stokes Equation (Pipe) & ${10^{-2}}^\dagger$ & 900 \\  \hline
    \end{tabular}
    \caption{Learning rate and maximum number of training iterations for all the problems. Learning rate with $(\cdot)^\dagger$ denotes the initial learning rate. For the  transonic flow and pipe flow test problems, we employed a step learning rate decay scheduler with a step size of 100 and a decay factor of 0.8.}
    \label{table:training_details}
\end{table}

\subsection{Elastic block problem}\label{Appendix:Elastic_block_problem}

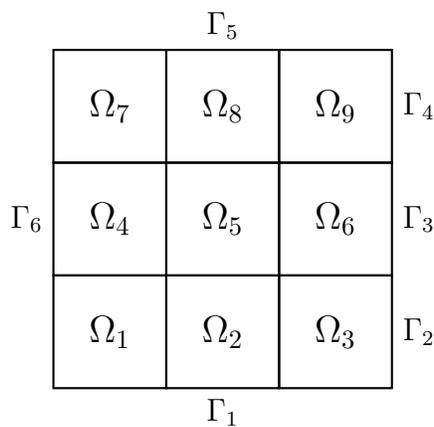
\begin{figure}[h]
    \centering
    \begin{tikzpicture}[scale=1.5]

        \foreach \x/\y/\num in {
            0/0/1, 1/0/2, 2/0/3,  
            0/1/4, 1/1/5, 2/1/6,  
            0/2/7, 1/2/8, 2/2/9   
        } {
            \draw[thick] (\x, \y) rectangle ++(1, 1);
            \node at (\x + 0.5, \y + 0.5) {\large $\Omega_{\num}$};
        }
    
        \node[below] at (1.5, 0) {$\Gamma_1$};
        
        \node[right] at (3, 0.5) {$\Gamma_2$};
        
        \node[right] at (3, 1.5) {$\Gamma_3$};
        
        \node[right] at (3, 2.5) {$\Gamma_4$};
        
        \node[above] at (1.5, 3) {$\Gamma_5$};
        
        \node[left] at (0, 1.5) {$\Gamma_6$};
    
    \end{tikzpicture}
    \caption{The subdomains and the boundary conditions for the linear elastic block problem~\cite{hesthaven2016certified} are shown.}
    \label{fig:elastic_block_png}
\end{figure}

This study focuses on a two-dimensional linear elasticity problem~\cite{hesthaven2016certified} in a square domain $\Omega$, segmented into nine square subdomains $\{\Omega_i\}_{i=1}^9$, as shown in Figure~\ref{fig:elastic_block_png}. The problem parameters include eight Young's moduli for the subdomains and three horizontal tractions on the right boundary. The Young's modulus ratio for each subdomain $\Omega_{p}$, $p \in \{1,2,\dots,8\}$, relative to the top-right subdomain $\Omega_9$, is denoted by $\mu_p \in [1, 100]$. Additionally, the horizontal tractions on the boundaries $\Gamma_{p-7}$, for $p \in \{9,10,11\}$, are defined by parameter $\mu_p \in [-1,1]$. The entire left boundary $\Gamma_6$ is clamped, while the top and bottom boundaries $\Gamma_1 \cup \Gamma_5$ remain traction-free. The parameter vector $\boldsymbol{\mu}$ is defined as $(\mu_1, \dots, \mu_{11})$ on the domain $\mathbb{P}=[1,100]^8\times[-1,1]^3$.

Training data is generated by sampling 6000 parameter vectors from the domain $\mathbb{P}$ using a Sobol sequence \cite{Sobol_1967}. For each parameter vector $\boldsymbol{\mu}$, RBniCS~\cite{RozzaBallarinScandurraPichi2024} is used to compute the displacement vector field $u(\boldsymbol{\mu}) = (u_1(\boldsymbol{\mu}), u_2(\boldsymbol{\mu}))$ on an unstructured grid using a finite element method. Of these, 5000 snapshots $\mathbf{u}_h(\boldsymbol{\mu})$ are used for training and the remaining 1000 for testing. The GP framework presented in Section~\ref{section:Struct-GPR} is employed to learn the mapping between parameter $\boldsymbol{\mu}$ and displacement fields $u_1$ and $u_2$. Training snapshots are first interpolated onto a $M=40\times 40$ Cartesian product grid. During testing, predicted snapshots are interpolated back onto the unstructured grid for error computation. To evaluate  performance, we use the relative error metric between the high-fidelity solution $\mathbf{u}_h(\boldsymbol{\mu}^{(i)})$ and the predicted snapshots $\widehat{\mathbf{u}}(\boldsymbol{\mu}^{(i)})$ for a given test parameter $\boldsymbol{\mu}^{(i)}$. Models are trained with an increasing number of snapshots ($N = 50, 100, 200, 500,$ $1000, 2000, 5000$) to evaluate their learning capacity with more training data.

Figure~\ref{fig:elastic_block_convergence} compares test errors obtained using different methods: our GP method, POD-GPR, and POD-Galerkin. Our method with the deep product kernel (DPK) shows significantly lower test errors than the data-driven POD-GPR. Furthermore, as the number of training snapshots $N$ increases, the test error of our GP method consistently decreases, eventually surpassing the accuracy of the physics-based POD-Galerkin, as shown in Figure~\ref{fig:elastic_block_convergence_sub2}. Figure~\ref{fig:elastic_block_box_plot_idx_0_1} shows the box plot of relative test errors on 1000 test parameters for models trained on an increasing number of snapshots ($N = 50, 100, 200, 500,$ $1000, 2000, 5000$). The box plots in Figure~\ref{fig:elastic_block_box_plot_idx_0_1} show that the test errors of the proposed GP method with DPK are significantly lower than the test errors of POD-GPR and POD-Galerkin when models are trained on a large number of training snapshots.

\begin{figure}[H]
  \centering
  \begin{subfigure}[b]{0.49\linewidth}
      \includegraphics[width=\linewidth]{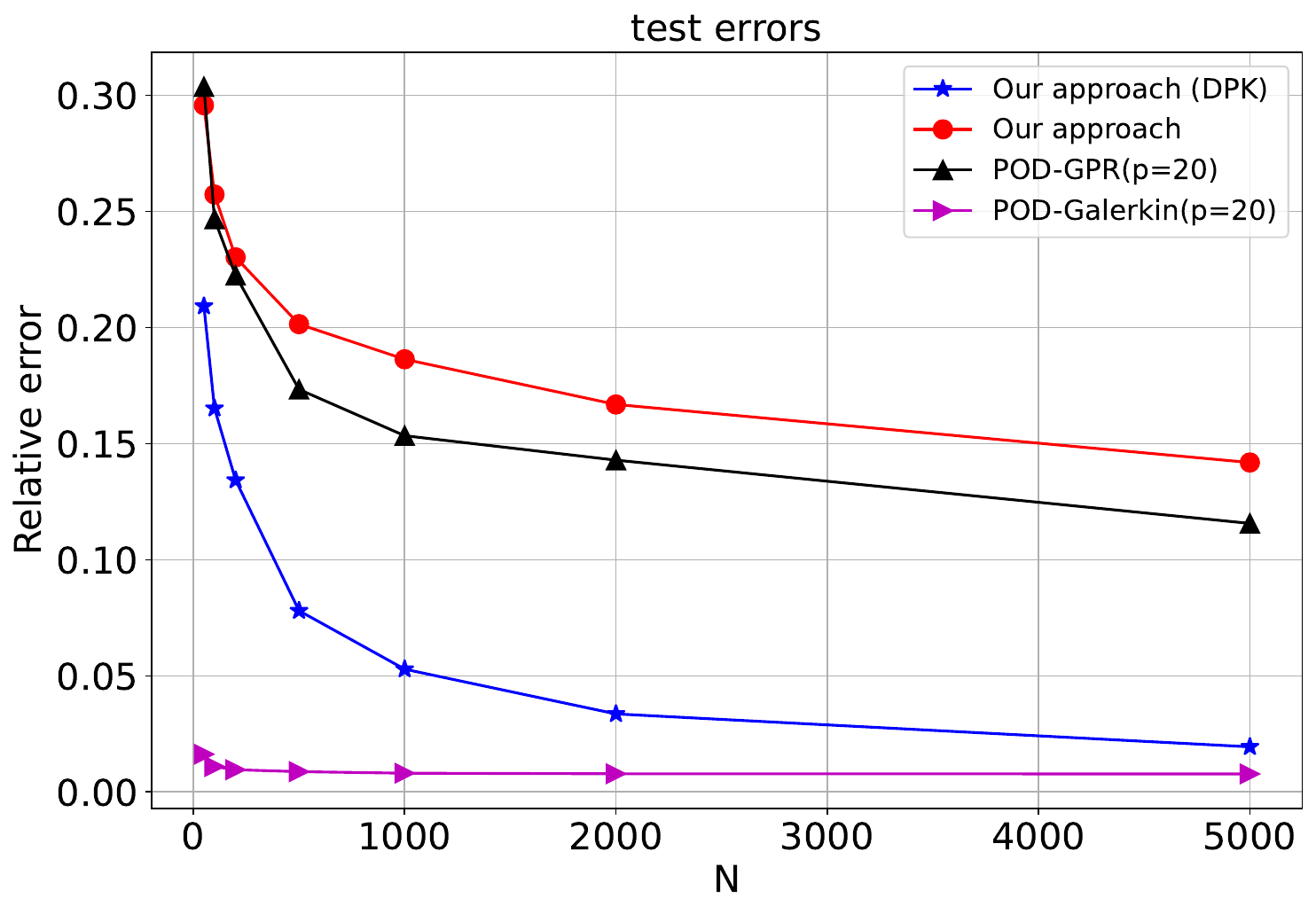}
      \caption{$u_1(\boldsymbol{\mu})$}
      \label{fig:elastic_block_convergence_sub1}
  \end{subfigure}
  \hfill
  \begin{subfigure}[b]{0.49\linewidth}
      \includegraphics[width=\linewidth]{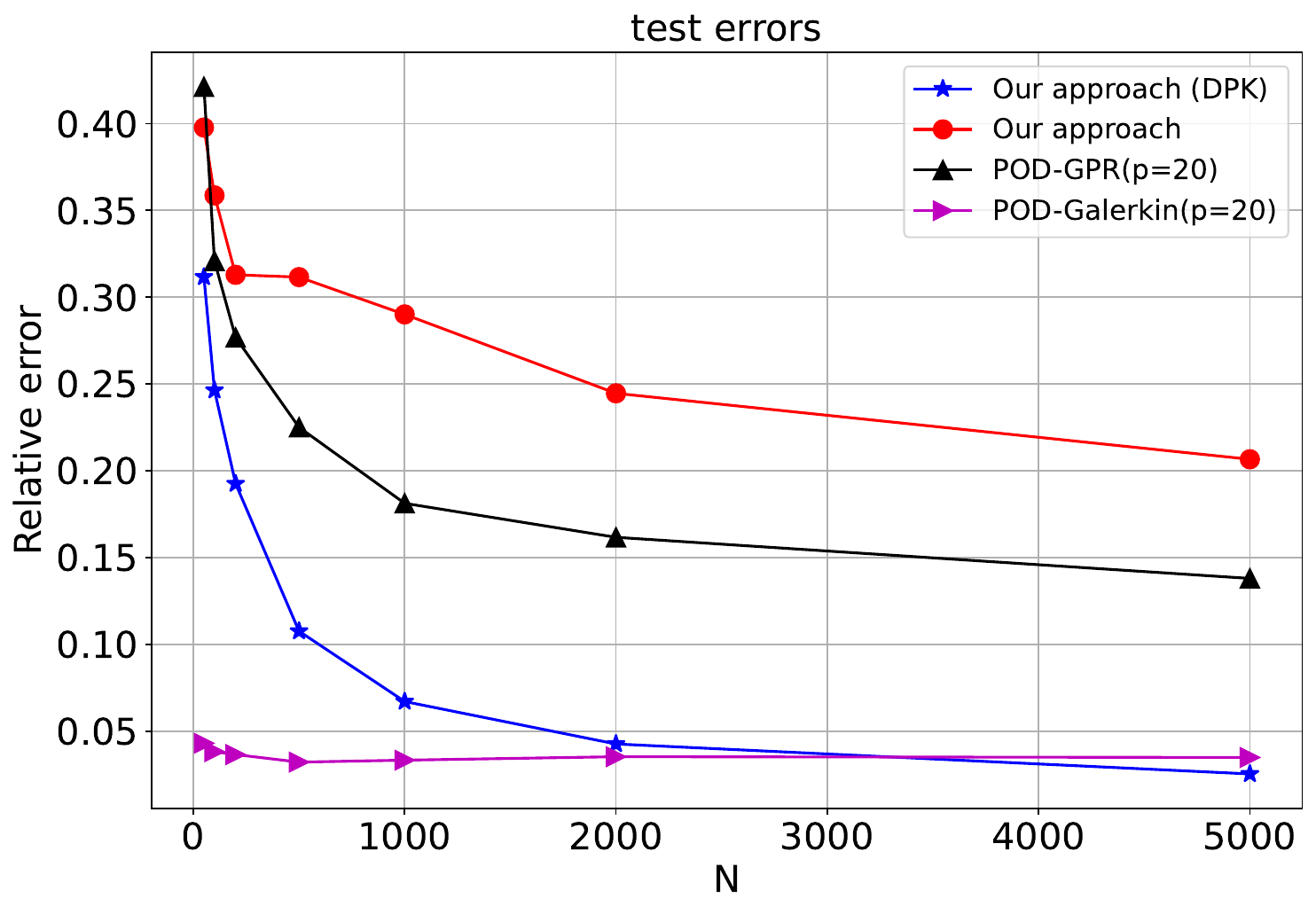}
      \caption{$u_2(\boldsymbol{\mu})$}
      \label{fig:elastic_block_convergence_sub2}
  \end{subfigure}
  \caption{\textit{Elastic block problem.} Convergence plot comparing the mean relative test errors corresponding to 1000 test snapshots for the proposed GP framework (using the deep product kernel (DPK)) with POD-GPR and physics-based POD-Galerkin with reduced dimension $p=20$.}
  \label{fig:elastic_block_convergence}
\end{figure}

\begin{figure}[H]
    \centering
    \begin{subfigure}[b]{0.49\linewidth}
        \includegraphics[width=\linewidth]{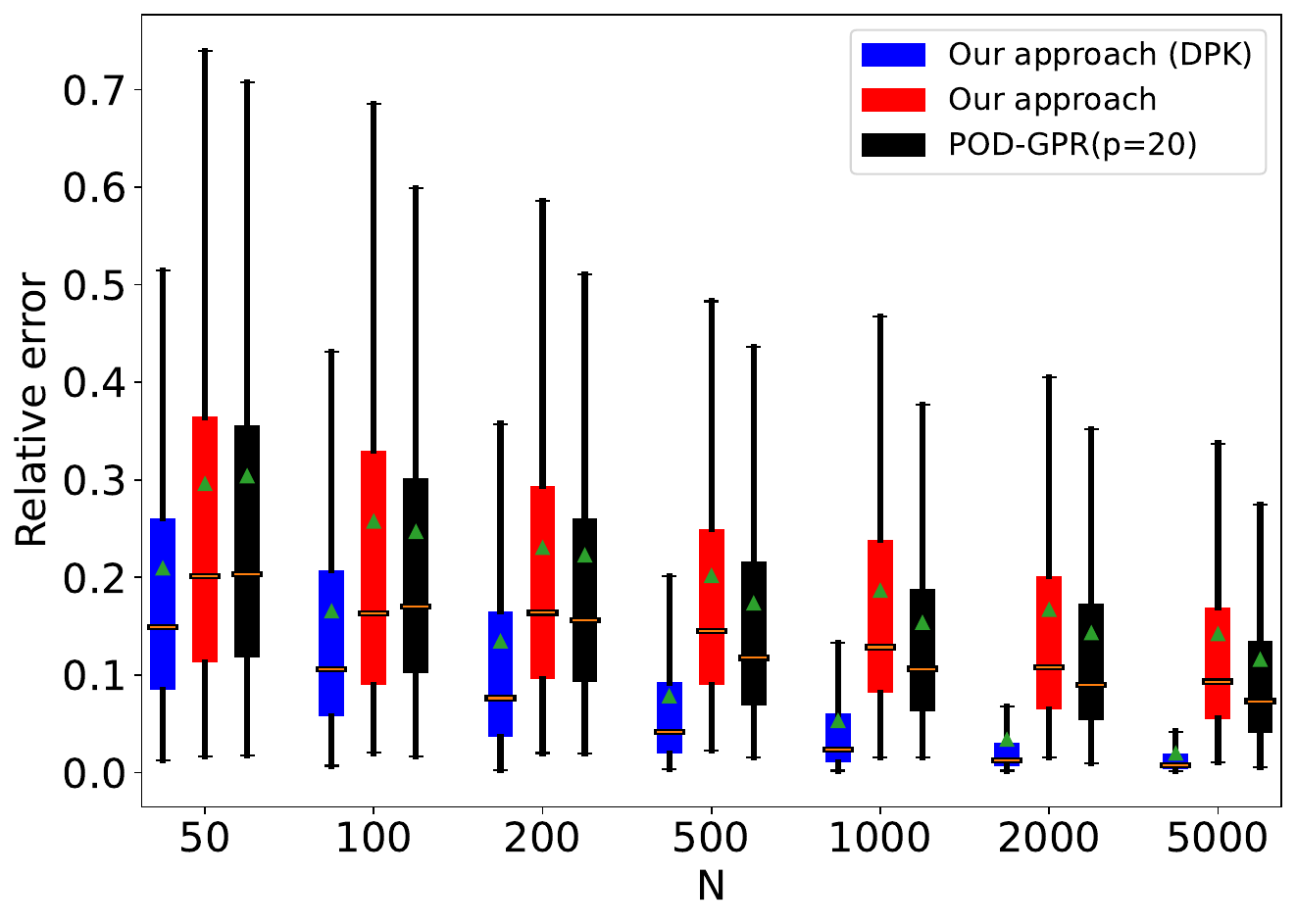}
        \caption{$u_1(\boldsymbol{\mu})$}
        \label{fig:elastic_block_box_plot_idx_0_sub1}
    \end{subfigure}
    \hfill
    \begin{subfigure}[b]{0.48\linewidth}
        \includegraphics[width=\linewidth]{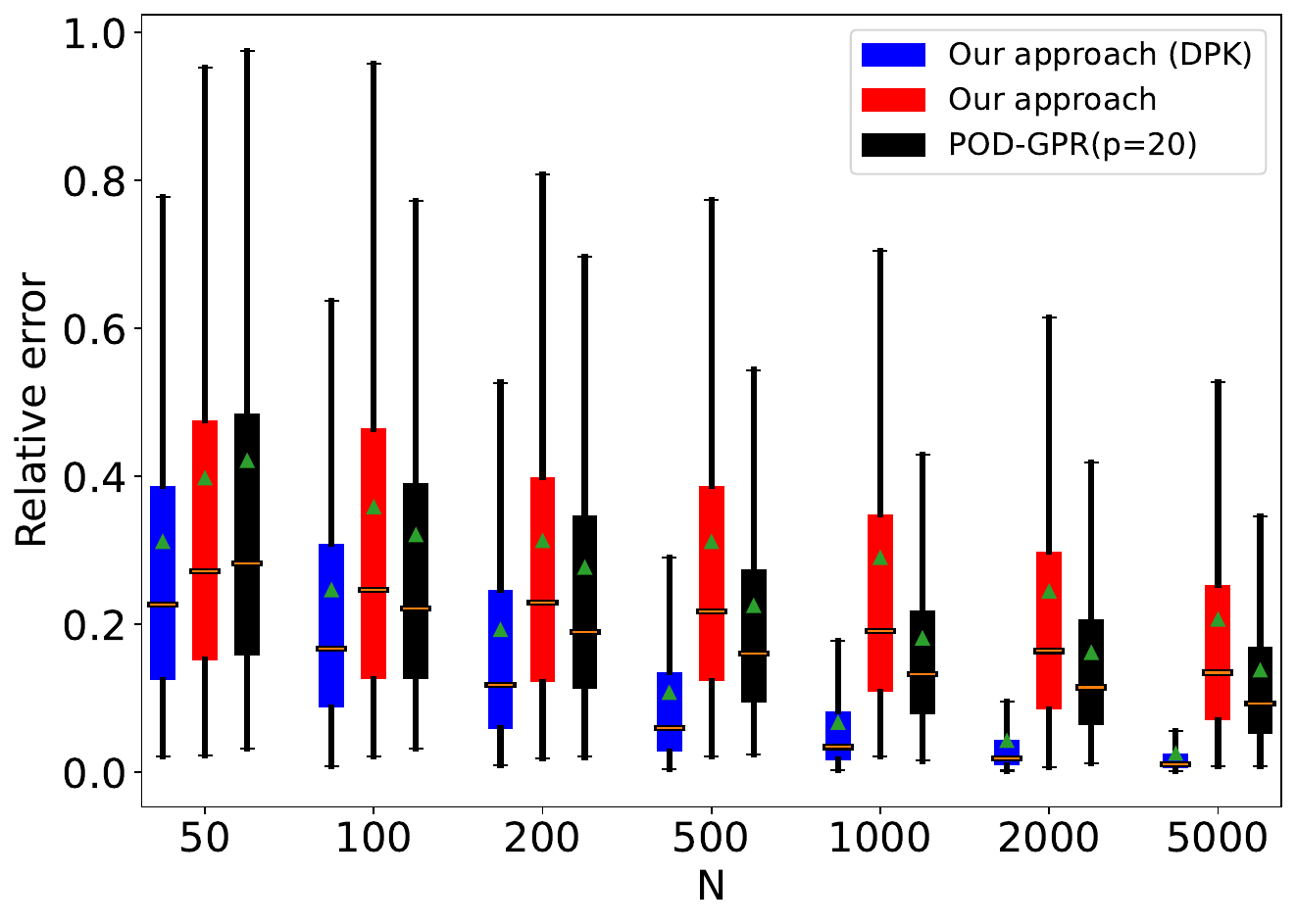}
        \caption{$u_2(\boldsymbol{\mu})$}
        \label{fig:elastic_block_box_plot_idx_1_sub2}
    \end{subfigure}
    \caption{\textit{Elastic block problem.} Box plot of relative test errors on 1000 test parameters for models trained on an increasing number of snapshots ($N = 50, 100, 200, 500,$ $1000, 2000, 5000$).}
    \label{fig:elastic_block_box_plot_idx_0_1}
  \end{figure}